\documentclass{article}
\usepackage{colt08e,times}
\usepackage{graphicx,multirow}
\usepackage{fullpage,amssymb,amsmath}
\let\hat\widehat
\let\tilde\widetilde
\def\fatnorm#1{|\kern-.2ex|\kern-.2ex| #1 |\kern-.2ex|\kern-.2ex|}
\newcommand{\twonorm}[1]{\left\lVert#1\right\rVert_2}
\newcommand{\fnorm}[1]{\lVert#1\rVert_F}
\newcommand{\norm}[1]{\left\lVert#1\right\rVert}
\newcommand{\abs}[1]{\left\lvert#1\right\rvert}

\newcommand{\R}{\bf R}
\newcommand{\Set}{\mathcal{S}}
\newcommand{\T}{\mathcal{T}}
\newcommand{\V}{\mathcal{V}}

\newcommand{\Lebesgue}{\mathcal{L}}
\newcommand{\half}{\ensuremath{\frac{1}{2}}}
\newcommand{\inv}[1]{\frac{1}{#1}}
\newcommand{\prob}[1]{\ensuremath{{\bf P}\left(#1\right)}}
\newcommand{\size}[1]{\ensuremath{\left|#1\right|}}
\newcommand{\expct}[1]{\ensuremath{{\bf E}#1}}
\newcommand{\silent}[1]{}
\newtheorem{remark}[theorem]{Remark}
\newtheorem{assumption}{A\!}
\newcommand{\be}{\begin{equation}}
\newcommand{\te}{\end{equation}}
\newcommand{\ve}{\varepsilon}
\newcommand{\se}{\!\scriptscriptstyle\searrow}
\newcommand{\ul}{\underline}
\def\qed{\hskip1pt $\;\;\scriptstyle\Box$}
\def\supp{\mathop{\text{supp}\kern.2ex}}
\def\prec{\mathop{\text{precision}\kern.2ex}}
\def\recall{\mathop{\text{recall}\kern.2ex}}
\def\cov{\mathop{\text{cov}\kern.2ex}}
\def\var{\mathop{\text{Var}\kern.2ex}}
\def\ess{\mathop{\text{ess}\kern.2ex}}

\newcommand{\func}[1]{\ensuremath{\mathrm{#1}}}
\newcommand{\mvec}{\func{vec}}
\newcommand{\tr}{\func{tr}}

\begin{document}

\title{Time Varying Undirected Graphs}

\author{Shuheng Zhou, John Lafferty and Larry Wasserman\thanks{This 
research was supported in part by NSF
grant CCF-0625879.  SZ thanks Alan Frieze and Giovanni Leoni
for helpful discussions on sparsity and smoothness of functions.
We thank J. Friedman, T. Hastie and R. Tibshirani for making GLASSO
publicly available, and anonymous reviewers for their constructive comments.}
\\
Carnegie Mellon University\\
{\tt \{szhou, lafferty\}@cs.cmu.edu, larry@stat.cmu.edu}}

\maketitle

\begin{abstract}
Undirected graphs are often used to describe high dimensional distributions.
Under sparsity conditions, the graph can be estimated using $\ell_1$ penalization
methods. However, current methods assume that the data are independent
and identically distributed. If the distribution, and hence the graph, 
evolves over time
then the data are not longer identically distributed.
In this paper, we show how to estimate
the sequence of graphs for non-identically
distributed data, where the distribution evolves
over time.
\end{abstract}
\section{Introduction}

Let $Z=(Z_1,\ldots, Z_p)^T$ be a random vector with distribution $P$.  The
distribution can be represented by an undirected graph $G=(V,F)$.  The vertex
set $V$ has one vertex for each component of the vector $Z$.  The edge set $F$
consists of pairs $(j,k)$ that are joined by an edge.  
If $Z_j$ is independent of $Z_k$ given the other variables, then $(j,k)$ is not in 
$F$.
When $Z$ is Gaussian, missing edges correspond to zeroes in the inverse 
covariance matrix $\Sigma^{-1}$.
Suppose we have independent, identically distributed data $D = (Z^1, \ldots,
Z^t, \ldots, Z^n)$ from $P$.  When $p$ is small, the graph may be estimated
from $D$ by testing which partial correlations are not significantly different
from zero~\cite{DP04}.  When $p$ is large, estimating $G$ is much more
difficult.  However, if the graph is sparse and the data are Gaussian, then
several methods can successfully estimate $G$; see~\cite{MB06, BGD07, FHT07,
LF07, BL07, RBLZ07}. 

All these methods assume that the graphical structure is stable over time.  But
it is easy to imagine cases where such stability would fail.  For example,
$Z^t$ could represent a large vector of stock prices at time $t$.  The
conditional independence structure between stocks could easily change over
time.  Another example is gene expression levels.  As a cell moves through its
metabolic cycle, the conditional independence relations between proteins could change.

In this paper we develop a nonparametric method for estimating time
varying graphical structure for multivariate Gaussian distributions using 
$\ell_1$ regularization method.
We show that, as long as the covariances
change smoothly over time, we can estimate the covariance matrix well
(in predictive risk) even when $p$ is large. We make the following
theoretical contributions: (i) nonparametric predictive risk
consistency and rate of convergence of the covariance matrices, (ii)
consistency and rate of convergence in Frobenius norm of the inverse
covariance matrix, (iii) large deviation results for covariance
matrices for non-identically distributed observations, and (iv) conditions
that guarantee smoothness of the covariances. In addition, we
provide simulation evidence that we can recover graphical structure. We
believe these are the first such results on time varying undirected
graphs.

\section{The Model and Method}

Let $Z^t \sim N(0,\Sigma(t))$  be independent.  It will be useful to index time as
$t= 0, 1/n, 2/n, \ldots , 1$ and thus the data are $D_n=(Z^t:\ t=0,1/n,\ldots,
1)$.  Associated with each each $Z^t$ is its undirected graph $G(t)$. Under the
assumption that the law ${\cal L}(Z^t)$ of $Z^t$  changes smoothly, we estimate
the graph sequence $G(1), G(2),\ldots,$.  The graph $G(t)$ is determined by the
zeroes of $\Sigma(t)^{-1}$.  
This method can be used to investigate a
simple time series model of the form: $W^0 \sim N(0,\Sigma(0)),$ and $$W^t =
W^{t-1} + Z^t, \; \; \; \text{where }\; \; Z^t \sim N(0,\Sigma(t)).$$
Ultimately, we are interested in the general time series model
where the $Z^t$'s are dependent and the graphs change over time.
For simplicity, however, we assume independence but allow the
graphs to change.
Indeed, it is the changing graph, rather than the dependence,
that is the biggest hurdle to deal with.

In the iid case, recent work \cite{BGD07,FHT07} 
has considered $\ell_1$-penalized maximum likelihood estimators
over the entire set of positive definite matrices,
\begin{equation}
\label{eq::iid-emp-intro}
{\hskip-.2cm}\hat \Sigma_{n} =  \arg \min_{\Sigma \succ 0} 
\big\{
{\rm tr}(\Sigma^{-1}\hat{S}_n) + \log|\Sigma| +
\lambda |\Sigma^{-1}|_1\big\}
\end{equation}
where $\hat{S}_n$ is the sample covariance matrix.
In the non-iid case our approach is to estimate $\Sigma(t)$ at time
$t$ by 
\begin{gather*}
{\hskip-.05cm}\hat \Sigma_{n}(t) =  \arg \min_{\Sigma \succ 0} 
\big\{
\tr(\Sigma^{-1}\hat{S}_n(t)) + \log |\Sigma| + \lambda |\Sigma^{-1}|_1\big\}
\end{gather*}
\begin{equation}
\label{eq:weightedcov}
\text{where } \; \;   \hat{S}_n(t) = 
\frac{\sum_s w_{st} Z_s Z_s^T }
    {\sum_s w_{st}}
\end{equation}
is a weighted covariance matrix, with weights $w_{st} =
K\left(\frac{|s-t|}{h_n}\right)$
given by a symmetric nonnegative function kernel over time; in other
words, $\hat{S}_n(t)$ is just the kernel estimator of the covariance
at time $t$.  An attraction of this
approach is that it can use existing software for covariance
estimation in the iid setting.

\subsection{Notation}
We use the following notation throughout the rest of the paper.  For any
matrix $W = (w_{ij})$, let $|W|$ denote the determinant of $W$, ${\rm tr}(W)$
the trace of $W$.
Let $\varphi_{\max}(W)$ and $\varphi_{\min}(W)$ be the
largest and smallest eigenvalues, respectively. 
We write $W^{\se} = 
\mbox{diag} (W)$ for a diagonal matrix with the same diagonal as $W$, and
$W^{\Diamond} = W - W^{\searrow}$. The matrix Frobenius norm is given by
$\norm{W}_F = \sqrt{\sum_i\sum_j w_{ij}^2}$. The operator norm
$\twonorm{W}^2$ is given by $\varphi_{\max}(WW^T)$.  We write $| \cdot |_1$
for the $\ell_1$ norm of a matrix vectorized, i.e., for a matrix
$|W|_1 = \norm{\mvec W}_1 = \sum_{i}\sum_j |w_{ij}|$, and write $\norm{W}_0$
for the number of non-zero entries in the matrix.  We use $\Theta(t) =
\Sigma^{-1}(t)$.

\section{Risk Consistency}

In this section we define the loss and risk.  Consider estimates
$\hat\Sigma_n(t)$ and $\hat{G}_n(t) = (V, \hat{F}_n)$. The first risk function is
\begin{equation}
\label{eq::model-risk}
U(G(t),\hat{G}_n(t)) = \expct{L(G(t),\hat{G}_n(t))}
\end{equation}
where
$L(G(t),\hat{G}_n(t)) = \abs{F(t)\; \Delta \; \hat{F}_n(t)}$, that is,
the size of the symmetric difference between two edge sets.
We say that $\hat{G}_n(t)$ is {\em sparsistent} if
$U(G(t),\hat{G}_n(t))\stackrel{P}{\to} 0$ as $n\to\infty$.

The second risk is defined as follows.  Let $Z\sim
N(0,\Sigma_0)$ and let $\Sigma$ be a positive definite matrix. Let
\begin{equation}
\label{eq::future-risk}
R(\Sigma) = {\rm tr}(\Sigma^{-1}\Sigma_0) + \log |\Sigma|.
\end{equation}
Note that, up to an additive constant, 
$$R(\Sigma) = -2 E_0(\log f_\Sigma(Z)),$$ 
where $f_\Sigma$ is the density for $N(0,\Sigma)$.
We say that $\hat{G}_n(t)$ is {\em persistent}~\cite{GR04} 
with respect to a class of positive definite matrices $\Set_n$ if
$R(\hat\Sigma_n) - 
\min_{\Sigma \in \Set_n} R(\Sigma)
\stackrel{P}{\to} 0.$
In the iid case, $\ell_1$ regularization yields a persistent
estimator, as we now show.  

The maximum likelihood estimate minimizes 
$$\hat{R}_n(\Sigma) = {\rm tr}(\Sigma^{-1}\hat{S}_n) + \log|\Sigma|,$$
where $\hat{S}_n$ is the sample covariance matrix.
Minimizing $\hat{R}_n(\Sigma)$ without constraints gives
$\hat\Sigma_n = \hat{S}_n$.
We would like to minimize 
$\hat{R}_n(\Sigma) \text{ subject to } \|\Sigma^{-1}\|_0 \leq L.$
This would give the ``best'' sparse graph $G$, 
but it is not a convex optimization problem.
Hence we estimate $\hat\Sigma_n$ by solving a convex relaxation problem 
as written in \eqref{eq::iid-emp-intro} instead.
Algorithms for carrying out this optimization are given 
by~\cite{BGD07, FHT07}. Given $L_n, \forall n$, let
\begin{equation}
\label{eq::magic-set}
\Set_n =\{\Sigma: \Sigma \succ 0, \abs{\Sigma^{-1}}_1 \leq L_n\}.
\end{equation}
We define the oracle estimator and write \eqref{eq::iid-emp-intro} as
\eqref{eq::emp-estimator}
\begin{eqnarray}
\label{eq::oracle-estimator}
\Sigma^*(n) & = & \arg\min_{\Sigma \in \Set_n} R(\Sigma), \\
\label{eq::emp-estimator}
\hat\Sigma_n & = & \arg\min_{\Sigma \in S_n} \hat{R}_n(\Sigma).
\end{eqnarray}
Note that one can choose to only penalize off-diagonal 
elements of $\Sigma^{-1}$ as in~\cite{RBLZ07}, if desired.
We have the following result, whose proof appears in 
Section~\ref{sec:thm-persis-proof}.
\begin{theorem}\label{lemma::iid}
Suppose that $p_n \leq n^\xi$ for some $\xi \geq 0$ and
$$
L_n = o\left(\frac{n}{\log p_n}\right)^{1/2}
$$
for~(\ref{eq::magic-set}). Then for the sequence of empirical estimators 
as defined in~(\ref{eq::emp-estimator}) and
$\Sigma^*(n), \forall n$ as in \eqref{eq::oracle-estimator},
\end{theorem}
$$
R(\hat\Sigma_n) - R(\Sigma^*(n)) \stackrel{P}{\to} 0.
$$
\subsection{Risk Consistency for the Non-identical Case}
In the non-iid case we estimate $\Sigma(t)$ at time $t \in [0, 1]$.
Given $\Sigma(t)$, let
$$
\hat{R}_n(\Sigma(t)) = {\rm tr}(\Sigma(t)^{-1}\hat{S}_n(t)) + \log |\Sigma(t)|.
$$
For a given $\ell_1$ bound $L_n$, we define $\hat{\Sigma}_n(t)$ as the minimizer of
$\hat{R}_n(\Sigma)$ subject to $\Sigma \in \Set_n$,
\begin{gather}
\label{eq::objective}
\hat{\Sigma}_n(t) =  \arg \min_{\Sigma \in \Set_n} \big\{
\tr(\Sigma^{-1}\hat{S}_n(t)) + \log |\Sigma|\big\}
\end{gather}
where $\hat{S}_n(t)$ is given in \eqref{eq:weightedcov},
with $K(\cdot)$ a symmetric 
nonnegative function with compact support:
\begin{assumption}
\label{as::kernel}
The kernel function $K$ has a bounded support $[-1, 1]$.
\end{assumption}
\begin{lemma}
\label{lemma:common}
Let $\Sigma(t) = \left[\sigma_{jk}(t)\right]$. Suppose the following conditions hold:
\begin{enumerate}
\item 
\label{lemma:cond::derivitives}
There exists $C_0>0, C$ such that $\max_{j,k}\sup_t |\sigma_{jk}'(t)|$ 
$\leq C_0$
and
$\max_{j,k}\sup_t |\sigma_{jk}''(t)| \leq C$.
\item $p_n \leq n^\xi$ for some $\xi \geq 0$.
\item
\label{lemma:cond::h}
$h_n \asymp n^{-1/3}$.
\end{enumerate}
Then
$\max_{j,k}|\hat{S}_n(t,j,k) - \Sigma(t,j,k)|  
= O_P\left(\frac{\sqrt{\log n}}{n^{1/3}}\right)$
for all $t>0$.
\end{lemma}
\begin{proof}
By the triangle inequality,
$$
|\hat{S}_n(t, j,k) - \Sigma(t, j,k)| \leq
|\hat{S}_n(t,j,k) - \expct{\hat{S}_n(t, j,k)}| +
$$
$$
|\expct{\hat{S}_n(t, j,k)} - \Sigma(t, j,k)|.
$$
In Lemma \ref{lemma::bias} we show that
\begin{equation*}
\label{eq::risk-bias}
\max_{j,k}\sup_t |\expct{\hat{S}_n(t, j,k)} - \Sigma(t, j,k)| = O(C_0 h_n).
\end{equation*}
In Lemma~\ref{lemma:deviation}, we show that
$$
\prob{|\hat{S}_n(t, j,k) - \expct{\hat{S}_n(t, j,k))}| > \epsilon}
\leq \exp\left\{ - c_1h_n n \epsilon^2 \right\}
$$
for some $c_1>0$.
Hence,
\begin{eqnarray} \nonumber
\label{eq::exact-tail}
\prob{\max_{j,k} |\hat{S}_n(t, j,k) 
- \expct{\hat{S}_n(t, j,k)}| > \epsilon}
&\leq& \\ 
\exp\left\{ - nh_n ( C \epsilon^2 - 2\xi \log n /(nh_n))\right\} \; \mbox { and}
\end{eqnarray}
$\max_{j,k} |\hat{S}_n(t, j,k) - \expct{\hat{S}_n(t, j,k)}|  = 
O_P\left(\sqrt{\frac{\log n}{nh_n}}\right).$
Hence the result holds for $h_n \asymp n^{-1/3}$.
\end{proof}

With the use of
Lemma~\ref{lemma:common}, the proof of the following follows the same
lines as that of Theorem \ref{lemma::iid}.
\begin{theorem}
\label{thm:persistence}
Suppose all conditions in Lemma~\ref{lemma:common} and the following hold:
\begin{gather}
\label{thm:cond::Ln}
L_n = o\left(n^{1/3}/\sqrt{\log n}\right).
\end{gather}
Then, $\forall t>0$, for the sequence of estimators as in~(\ref{eq::objective}),
$$R(\hat\Sigma_n(t)) -  R(\Sigma^*(t))
\stackrel{P}{\to} 0.$$
\end{theorem}

\begin{remark} If a local linear smoother is substituted for a kernel
smoother, the rate can be improved from $n^{1/3}$ to $n^{2/5}$ as the bias
will be bounded as $O(h^2)$ in \eqref{eq::risk-bias}.
\end{remark}

\begin{remark}
Suppose that $\forall i, j$, if $\theta_{ij} \not= 0$, we have 
$\theta_{ij} = \Omega(1)$. Then Condition \eqref{thm:cond::Ln} allows that
$\abs{\Theta}_1 = L_n$; hence if $p = n^{\xi}$ and $\xi < 1/3$, 
we have that $\norm{\Theta}_0 = \Omega(p)$. Hence the family of graphs 
that we can guarantee persistency for, although sparse, is likely to include 
connected graphs, for example, when $\Omega(p)$ edges were formed randomly 
among $p$ nodes. 
\end{remark}
The smoothness condition in Lemma~\ref{lemma:common} is expressed in terms
of the elements of $\Sigma(t) = \left[\sigma_{ij}(t)\right]$.
It might be more natural to impose smoothness on 
$\Theta(t) = \Sigma(t)^{-1}$ instead. In fact, 
smoothness of $\Theta_t$ implies smoothness of $\Sigma_t$ as 
the next result shows.
Let us first specify two assumptions.
We use $\sigma^2_{i}(x)$ as a shorthand for $\sigma_{ii}(x)$. 
\begin{definition}
\label{def:sup-norms}
For a function $u: [0, 1] \rightarrow \R$, let
$\norm{u}_{\infty} = \sup_{x \in [0, 1]}|u(x)|$.
\end{definition}
\begin{assumption}
\label{as::S_0}
There exists some constant $S_0 < \infty$ such that
\begin{eqnarray}
\max_{i=1 \ldots, p} \sup_{t \in [0, 1]}|\sigma_{i}(t)| & \leq & S_0 < \infty, 
\; \mbox{ hence } \\
\max_{i=1 \ldots, p}\norm{\sigma_i}_{\infty} & \leq & S_0.
\end{eqnarray}
\end{assumption}
\begin{assumption}
\label{as::Theta}
Let $\theta_{ij}(t), \forall i, j,$ be twice differentiable functions such that 
$\theta'_{ij}(t) < \infty$ and $\theta''_{i j}(t) < \infty, \forall t \in [0,1]$.
In addition, there exist constants $S_1, S_2 < \infty$ such that
\begin{eqnarray}
\sup_{t \in [0, 1]}
\sum_{k=1}^p \sum_{\ell=1}^p \sum_{i=1}^p \sum_{j=1}^p 
|\theta_{ki}'(t)\theta_{\ell j}'(t)| & \leq & S_1 \\
\sup_{t \in [0, 1]}\sum_{k=1}^p \sum_{\ell=1}^p |\theta''_{k \ell}(t)|
& \leq & S_2,
\end{eqnarray}
where the first inequality guarantees that

$\sup_{t \in [0, 1]}\sum_{k=1}^p \sum_{\ell=1}^p
|\theta_{k\ell}'(t)| < \sqrt{S_1} < \infty.$
\end{assumption}
\begin{lemma}
Denote the elements of
$\Theta(t) = \Sigma(t)^{-1}$ by $\theta_{jk}(t)$. 
Under A~\ref{as::S_0} and A~\ref{as::Theta}, the smoothness condition
in Lemma~\ref{lemma:common} holds.
\end{lemma}
The proof is in Section~\ref{sec:sparsity}.
In Section~\ref{sec:append-implications}, we show some preliminary results on 
achieving upper bounds on quantities that appear in 
Condition~\ref{lemma:cond::derivitives} of Lemma~\ref{lemma:common} through
the sparsity level of the inverse covariance matrix, i.e., 
$\norm{\Theta_t}_0, \forall t \in [0, 1]$.

\subsection{Proof of Theorem~\ref{lemma::iid}}
\label{sec:thm-persis-proof}
Note that $\forall n$,
$\sup_{\Sigma\in \Set_n} |R(\Sigma) - \hat{R}_n(\Sigma)| 
 \leq $
\begin{eqnarray*}
\sum_{j,k} | \Sigma^{-1}_{jk}|\ |\hat{S}_n(j,k)-\Sigma_{0}(j,k)|
 \leq
\delta_n \ \abs{\Sigma^{-1}}_1,
\end{eqnarray*}
where it follows from~\cite{RBLZ07} that
$$
\delta_n = \max_{j,k} |\hat{S}_n(j,k)-\Sigma_0(j,k)| =O_P(\sqrt{\log p/n}).
$$
Hence, minimizing over $\Set_n$ with 
$L_n = o\left(\frac{n}{\log p_n}\right)^{1/2}$,
\noindent
$\sup_{\Sigma\in \Set_n} |R(\Sigma) - \hat{R}_n(\Sigma)| = o_P(1).$
By the definitions of $\Sigma^*(n) \in \Set_n$ and  $\hat\Sigma_n \in \Set_n$,
we immediately have
$R(\Sigma^*(n)) \leq R(\hat\Sigma_n)$ and 
$\hat{R}_n(\hat\Sigma_n) \leq \hat{R}_n(\Sigma^*(n))$; thus
\begin{eqnarray*}
0 &\leq & R(\hat\Sigma_n) - R(\Sigma^*(n)) \\
& = &
R(\hat\Sigma_n) - \hat{R}_n(\hat\Sigma_n) + \hat{R}_n(\hat\Sigma_n) - R(\Sigma^*(n)) \\
& \leq &
R(\hat\Sigma_n) - \hat{R}_n(\hat\Sigma_n) + \hat{R}_n(\Sigma^*(n)) - R(\Sigma^*(n))
\end{eqnarray*}
Using the triangle inequality and $\hat\Sigma_n, \Sigma^*(n) \in \Set_n$, 
\begin{eqnarray*}
\lefteqn{
|R(\hat\Sigma_n) - R(\Sigma^*(n))|\leq } \\
 & & |R(\hat\Sigma_n) - \hat{R}_n(\hat\Sigma_n) 
+ \hat{R}_n(\Sigma^*(n)) - R(\Sigma^*(n)) | \\
&\leq & 
|R(\hat\Sigma_n) - \hat{R}_n(\hat\Sigma_n)| + | \hat{R}_n(\Sigma^*(n)) - R(\Sigma^*(n)) | \\
& \leq &
2 \sup_{\Sigma\in \Set_n}|R(\Sigma) - \hat{R}_n(\Sigma)|.
\; \; \mbox{Thus } \forall \epsilon > 0,
\end{eqnarray*}
  the event 
$\left\{\abs{R(\hat\Sigma_n) - R(\Sigma^*(n))} > \epsilon \right\}$
is contained in the event
$\left\{\sup_{\Sigma\in \Set_n}|R(\Sigma) - \hat{R}_n(\Sigma)| > \epsilon/2\right\}$.
Thus, for $L_n = o( (n/\log n)^{1/2})$, and $\forall \epsilon > 0,$
as $n \to \infty$,
$$\prob{\abs{R(\hat\Sigma_n) - R(\Sigma^*(n)) }
> \epsilon} \leq $$
$\prob{\sup_{\Sigma\in \Set_n}|R(\Sigma) - \hat{R}_n(\Sigma)| > \epsilon/2} \to 0.
\; \; \; \Box$

\section{Frobenius Norm Consistency}
\label{sec:frob}
In this section, we show an explicit convergence rate in the Frobenius norm for
estimating $\Theta(t), \forall t$, where $p, |F|$ grow with $n$, so long as
the covariances change smoothly over $t$.
Note that certain smoothness assumptions on a matrix $W$ would guarantee the 
corresponding smoothness conditions on its inverse $W^{-1}$, so long as $W$ 
is non-singular, as we show in Section~\ref{sec:sparsity}.
We first write our time-varying estimator $\hat \Theta_n(t)$ for
$\Sigma^{-1}(t)$ at time $t \in [0,1]$ 
as the minimizer of the $\ell_1$ regularized negative smoothed log-likelihood
over the entire set of positive definite matrices,
\begin{equation}
\label{eq::lag-objective}
{\hskip-.3cm}\hat \Theta_{n}(t) =  \arg \min_{\Theta \succ 0} \big\{
{\rm tr}(\Theta\hat{S}_n(t)) - \log |\Theta| +
\lambda_n |\Theta|_{1} \big\}
\end{equation}
where $\lambda_n$ is a non-negative regularization parameter, and
$\hat{S}_n(t)$ is the smoothed sample covariance matrix using a kernel function as 
defined in~\eqref{eq:weightedcov}.

Now fix a point of interest $t_0$.
In the following, we use $\Sigma_0 = (\sigma_{ij}(t_0))$ to
denote the true covariance matrix at this time.
Let $\Theta_0 = \Sigma_0^{-1}$ be its inverse matrix.
Define the set $S =\{(i,j): \theta_{ij}(t_0) \neq 0, \ i \neq j\}$.
Then $|S| = s$.
Note that $|S|$ is twice the number of edges in the graph $G(t_0)$.
We make the following assumptions.
\begin{assumption}
\label{as::frob}
Let $p + s = o\left(n^{(2/3)}/\log n\right)$ and
$\varphi_{\min} (\Sigma_0) \geq \underline{k}> 0$, hence
$\varphi_{\max} (\Theta_0) \leq 1 / \underline{k}$.
For some sufficiently large constant $M$, let
$\varphi_{\min}(\Theta_0) = \Omega\left(2M\sqrt{ \frac{(p+s)\log n}{n^{2/3}}}\right)$.
\end{assumption}
The proof draws upon techniques from~\cite{RBLZ07}, with 
modifications necessary to handle the fact that we penalize $\abs{\Theta}_1$ 
rather than $|\Theta^{\Diamond}|_1$ as in their case.
\begin{theorem}
\label{thm:frob}
Let $\hat\Theta_{n}(t)$ be the minimizer defined by \eqref{eq::lag-objective}.
Suppose all conditions in Lemma~\ref{lemma:common} and 
A~\ref{as::frob} hold. If 
$$\lambda_n \asymp \sqrt{\frac{\log n}{n^{2/3}}}, \; \; \; \mbox{then}$$
\begin{equation}
\label{eq::frob-con}
\fnorm{\hat \Theta_{n}(t) -\Theta_0} =
 O_P\left(2M \sqrt{ \frac{(p+s)\log n}{n^{2/3}}}\right) \ .
\end{equation}
\end{theorem}
\begin{proof}
Let $\ul{0}$ be a matrix with all entries being zero. Let
\begin{eqnarray}
\label{eq::convex-func}
Q (\Theta) &  = & \nonumber \tr (\Theta \hat{S}_n(t_0)) - \log |\Theta| +
\lambda |\Theta| - \\
&& \nonumber \tr (\Theta_0 \hat{S}_n(t_0)) + \log |\Theta_0| - \lambda |\Theta_0|_1  \\
& = & \nonumber
\tr \left(( \Theta-\Theta_0) (\hat{S}_n(t)-\Sigma_0)\right)-\\
&& \nonumber (\log |\Theta| - \log |\Theta_0|)
+ \tr\left((\Theta-\Theta_0) \Sigma_0 \right) \\
&+& \lambda ( |\Theta|_1  - |\Theta_0|_1).
\end{eqnarray}
$\Hat{\Theta}$ minimizes $Q(\Theta)$, or equivalently 
$\hat\Delta_n = \hat{\Theta} - \Theta_0$ minimizes
$G(\Delta) \equiv Q(\Theta_0 + \Delta)$. 
Hence $G(\ul{0}) = 0$ and $G(\hat\Theta_n) \leq G(\ul{0}) = 0$ by definition.
Define for some constant $C_1$,
$\delta_n = C_1 \sqrt{\frac{\log n}{n^{2/3}}}.$
Now, let
\be 
\lambda_n = \frac{C_1}{\ve} \sqrt{ \frac{\log n}{n^{2/3}}}
= \frac{\delta_n}{\ve} \; \;\mbox{for some }\;\; 0 < \ve < 1.
\te
Consider now the set
\begin{equation*}
\T_n = \{ \Delta: \Delta = B - \Theta_0 , B, \Theta_0 \succ 0, \fnorm{\Delta} = M r_n\},
\end{equation*}
where
\be
\label{eq::sphere}
r_n = \sqrt{\frac{(p+s)\log n }{n^{2/3}}} \asymp \delta_n \sqrt{p+s} \to 0.
\te
\begin{claim}
\label{claim:posi-def-interval}
Under A~\ref{as::frob}, for all $\Delta \in \T_n$
such that  $\fnorm{\Delta} = o(1)$ as in \eqref{eq::sphere},
$\Theta_0 + v \Delta \succ 0, \forall v \in I \supset [0, 1]$.
\end{claim}
\begin{proof}
It is sufficient to show that $\Theta_0 + (1 + \ve) \Delta \succ 0$
and $\Theta_0 - \ve \Delta \succ 0$ for some $1 > \ve > 0$.
Indeed, $\varphi_{\min} (\Theta_0 + (1 + \ve) \Delta) \geq 
\varphi_{\min} (\Theta_0) - (1 + \ve) \twonorm{\Delta} > 0$
for $\ve < 1$, given that $\varphi_{\min} (\Theta_0) = \Omega(2Mr_n)$ and 
$\twonorm{\Delta} \leq \fnorm{\Delta} = Mr_n$.
Similarly, $\varphi_{\min} (\Theta_0 - \ve \Delta)  \geq
\varphi_{\min} (\Theta_0) - \ve \twonorm{\Delta} > 0$ for 
$\ve < 1$. 
\end{proof}

Thus we have that $\log\det(\Theta_0 + v \Delta)$ is infinitely differentiable on 
the open interval $I \supset [0, 1]$ of $v$. This allows us to use the Taylor's 
formula with integral remainder to obtain the following lemma:
\begin{lemma}
\label{lemma:sphere-pos-bound}
With probability $1 - 1/n^c$ for some $c \geq 2$,
$G(\Delta) > 0$ for all $\Delta \in \T_n$.
\end{lemma}
\begin{proof}
Let us use $A$ as a shorthand for
$$\mvec{\Delta}^T \left( \int^1_0(1-v)
(\Theta_0 + v \Delta)^{-1} \otimes (\Theta_0 + v \Delta)^{-1}dv
\right) \mvec{\Delta},$$
where $\otimes$ is the Kronecker product (if
$W= (w_{ij})_{m \times n}$, $P=(b_{k\ell})_{p \times q}$,
then $W \otimes P = (w_{ij}P)_{m p \times nq}$),
and $\mvec{\Delta} \in \R^{p^2}$ is $\Delta_{p \times p}$ 
vectorized. Now, the Taylor expansion gives \\
$\log|\Theta_0 + \Delta| - \log|\Theta_0| = 
\frac{d}{dv}\log|\Theta_0 + v\Delta||_{v=0} \Delta + \\
\int_0^1(1-v) \frac{d^2}{dv^2}  \log\det(\Theta_0 + v \Delta) dv 
= {\rm tr}(\Sigma_0 \Delta) + A,$
where by symmetry,
${\rm tr}(\Sigma_0 \Delta) = \tr (\Theta-\Theta_0) \Sigma_0$. Hence
\begin{eqnarray}
\label{eq::G2}
\lefteqn{G(\Delta) \;=\;} \\
\nonumber
&&  A +\tr \left( \Delta (\hat{S}_n -\Sigma_0) \right)
 + \lambda_n \left( |\Theta_0 + \Delta|_1  - |\Theta_0|_1\right).
\end{eqnarray}
For an index set $S$ and a matrix $W = [w_{ij}]$, write $W_S \equiv
(w_{ij} I( (i,j) \in S))$, where $I(\cdot)$ is an indicator function.
Recall $S= \{ (i,j) : \ \Theta_{0ij} \neq 0, \ i\neq j \}$ and let
$S^c = \{ (i,j) : \ \Theta_{0ij} = 0, \ i\neq j \}$.
Hence 
$\Theta =  \Theta^{\se} + \Theta_{S}^{\Diamond} + \Theta_{S^c}^{\Diamond}, 
\forall \Theta$ in our notation.
Note that we have $\Theta_{0S^c}^{\Diamond} = \ul{0}$,
\begin{eqnarray*} 
\nonumber
|\Theta_0^{\Diamond} + \Delta^{\Diamond}|_1 & = & |\Theta_{0S}^{\Diamond} 
+ \Delta_S^{\Diamond}|_1 +  |\Delta_{S^c}^{\Diamond}|_1, \\ \nonumber
|\Theta_0^{\Diamond}|_1 & = & |\Theta^{\Diamond}_{0S}|_1, \; \; 
\mbox{ hence} \\ \nonumber
|\Theta^{\Diamond}_0 + \Delta^{\Diamond}
|_1  - |\Theta_0^{\Diamond}|_1 
& \geq & 
\label{eq::diag-norm}
\abs{\Delta^{\Diamond}_{S^c}}_1 - \abs{\Delta^{\Diamond}_S}_1, \\
|\Theta^{\se}_0 + \Delta^{\se}|_1  -  |\Theta_0^{\se}|_1 
& \geq &   
- |\Delta^{\se}|_1,
\end{eqnarray*}
where the last two steps follow from the triangle inequality. Therefore
\begin{eqnarray}
\label{eq::Delta-I} \nonumber
\lefteqn{|\Theta_0 + \Delta|_1  - |\Theta_0|_1 = } \\ \nonumber
& &|\Theta^{\Diamond}_0 + \Delta^{\Diamond}|_1  - |\Theta_0^{\Diamond}|_1 + 
|\Theta^{\se}_0 + \Delta^{\se}|_1  -  |\Theta_0^{\se}|_1   \\
& \geq &   
\abs{\Delta_{S^c}^{\Diamond}}_1 - 
\abs{\Delta_{S}^{\Diamond}}_1 - |\Delta^{\se}|_1.
\end{eqnarray}
Now, from Lemma~\ref{lemma:common},
$\max_{j,k}|\hat{S}_n(t,j,k) - \sigma(t,j,k)| =
O_P\left(\frac{\sqrt{\log n}}{n^{1/3}}\right)$
$ = O_P(\delta_n).$
By \eqref{eq::exact-tail}, with probability $1 - \inv{n^2}$
\begin{eqnarray}  \nonumber
& & \abs{\tr \big(\Delta  (\hat{S}_n -\Sigma_0)\big)} 
\le \delta_n \abs{\Delta}_1, \; \; \mbox{ hence by \eqref{eq::Delta-I} \nonumber} \\  \nonumber
\lefteqn{\tr \left( \Delta (\hat{S}_n -\Sigma_0) \right) + 
\lambda_n \left( |\Theta_0 + \Delta|_1  - |\Theta_0|_1\right)}\\ \nonumber
& \geq &
\lefteqn{- \delta_n\abs{\Delta^{\se}}_1 -  
\delta_n \abs{\Delta_{S^c}^{\Diamond}}_1 -
 \delta_n \abs{\Delta_S^{\Diamond}}_1}\\ \nonumber
& & 
-\lambda_n |\Delta^{\se}|_1 + \lambda_n \abs{\Delta_{S^c}^{\Diamond}}_1 
-  \lambda_n\abs{\Delta_{S}^{\Diamond}}_1
\\ \nonumber
& \geq & - (\delta_n + \lambda_n) \left(|\Delta^{\se}|_1 + \abs{\Delta_{S}^{\Diamond}}_1\right)
+ (\lambda_n -\delta_n) \abs{\Delta_{S^c}^{\Diamond}}_1 \\
& \geq &
\label{eq::trace-term} 
- (\delta_n + \lambda_n) 
\left(|\Delta^{\se}|_1 + \abs{\Delta_{S}^{\Diamond}}_1\right), \mbox{ where}
\end{eqnarray}
\begin{eqnarray} \nonumber
\lefteqn{(\delta_n + \lambda_n) 
\left(|\Delta^{\se}|_1 + \abs{\Delta_{S}^{\Diamond}}_1\right)} \\ \nonumber
& \leq & 
(\delta_n + \lambda_n) 
\left(\sqrt{p}\fnorm{\Delta^{\se}} + \sqrt{s} \fnorm{\Delta_{S}^{\Diamond}}\right) 
\\   \nonumber
& \leq & 
(\delta_n + \lambda_n) 
\left(\sqrt{p}\fnorm{\Delta^{\se}} + \sqrt{s} \fnorm{\Delta^{\Diamond}}\right) 
\\   \nonumber
& \leq & 
(\delta_n + \lambda_n) \max\{\sqrt{p}, \sqrt{s}\} 
\left(\fnorm{\Delta^{\se}} + \fnorm{\Delta^{\Diamond}}\right)
\\   \nonumber
& \leq & 
(\delta_n + \lambda_n) \max\{\sqrt{p}, \sqrt{s}\}\sqrt{2}\fnorm{\Delta} \\ 
& \leq & 
\label{eq::G2-two-terms}
\delta_n\frac{1 + \ve}{\ve} \sqrt{p + s}\sqrt{2} \fnorm{\Delta}.
\end{eqnarray}
Combining \eqref{eq::G2}, \eqref{eq::trace-term}, and \eqref{eq::G2-two-terms}, 
we have with probability $1 - \inv{n^c}$, for all $\Delta \in \T_n$,
\begin{eqnarray*}
\nonumber
G(\Delta) & \geq & A - (\delta_n + \lambda_n) 
\left(|\Delta^{\se}|_1 + \abs{\Delta_{S}^{\Diamond}}_1\right) \\
& \geq &  \frac{\underline{k}^2}{2 + \tau} \| \Delta \|^2_F
-  \delta_n\frac{1 + \ve}{\ve} \sqrt{p + s}\sqrt{2} \fnorm{\Delta}  \\
&= &
\fnorm{\Delta}^2\left(\frac{\underline{k}^2}{2 + \tau} 
- \delta_n\frac{\sqrt{2}(1 + \ve)}{\ve\fnorm{\Delta}} \sqrt{p + s} \right) \\
& = & 
\fnorm{\Delta}^2\left(\frac{\underline{k}^2}{2 + \tau} 
- \frac{\delta_n\sqrt{2} (1 + \ve)}{\ve M r_n}\sqrt{p + s}\right) > 0
\end{eqnarray*}
for $M$ sufficiently large, where the bound on $A$ comes 
from Lemma~\ref{lemma:integral} by~\cite{RBLZ07}.
\end{proof}
\begin{lemma}{\textnormal(\cite{RBLZ07})}
\label{lemma:integral}
For some $\tau = o(1)$, under A~\ref{as::frob},\\
$\mvec{\Delta}^T \left(\int^1_0(1-v)
(\Theta_0 + v \Delta)^{-1} \otimes (\Theta_0 + v \Delta)^{-1}dv
\right)\mvec{\Delta} \\
\mbox{     }\geq \fnorm{\Delta}^2 \frac{\underline{k}^2}{2 + \tau}$,
for all $\Delta \in \T_n$.
\end{lemma}
We next show the following claim.
\begin{claim}
\label{claim:outside}
If $G(\Delta) > 0, \forall \Delta \in \T_n$,
then $G(\Delta) > 0$ for all $\Delta$ in
$\V_n = 
\{\Delta: \Delta = D - \Theta_0, D \succ 0, \fnorm{\Delta} > M r_n,
\mbox{ for }r_n \mbox{ as in }\eqref{eq::sphere}\}$.
Hence if $G(\Delta) > 0, \forall \Delta \in \T_n$,
then $G(\Delta) > 0$ for all $\Delta \in \T_n \cup \V_n$.
\end{claim}
\begin{proof}
Now by contradiction, suppose $G(\Delta') \leq 0$ for some $\Delta' \in \V_n$.
Let $\Delta_0 = \frac{M r_n}{\fnorm{\Delta'}}\Delta'$. 
Thus $\Delta_0 = \theta \ul{0} + (1 - \theta) \Delta'$, where $0< 1 - \theta = 
 \frac{M r_n}{\fnorm{\Delta'}} <1$ by definition of $\Delta_0$.
Hence $\Delta_0 \in \T_n$ given that $\Theta_0 + \Delta_0 \succ 0$ by
Claim~\ref{claim:posi-def}.
Hence by convexity of $G(\Delta)$, we have that 
$G(\Delta_0) \leq  \theta G(\ul{0}) + (1 - \theta) G(\Delta') \leq 0$,
contradicting that $G(\Delta_0) > 0$ for $\Delta_0 \in \T_n$.
\end{proof}

By Claim~\ref{claim:outside} and the fact that 
$G(\hat\Delta_n) \leq G(0) =0$, we have the following: 
If $G(\Delta) > 0, \forall \Delta \in \T_n$, then
$\hat\Delta_n \not\in (\T_n \cup \V_n)$, that is,
$\fnorm{\hat\Delta_n} < M r_n$, given that 
$\hat\Delta_n = \hat\Theta_n - \Theta_0$,
where $\hat\Theta_n, \Theta_0 \succ 0$. Therefore
\begin{eqnarray*}
\label{eq::end-of-proof}
\lefteqn{\prob{\| \hat\Delta_n \|_F  \geq M r_n} =  
1 - \prob{\| \hat\Delta_n \|_F < M r_n}}  \\
& \leq & 1 - \prob{G(\Delta) > 0, \forall \Delta \in \T_n} \\
& = & \prob{G(\Delta) \leq 0 \text{ for some } \Delta \in \T_n}  < \inv{n^c}.
\end{eqnarray*}
We thus establish that $\fnorm{\hat\Delta_n} \leq O_P\left(M r_n\right)$.
\end{proof}
\begin{claim}
\label{claim:posi-def}
Let $B$ be a $p \times p$ matrix.
If $B \succ 0$ and  $B + D \succ 0$, 
then  $B + v D \succ 0$ for all $v \in [0, 1]$.
\end{claim}
\begin{proof}
We only need to check for $v \in (0, 1)$, where $1 - v > 0$; 
$\forall x \in \R^p$, by $B \succ 0$ and $B + D \succ 0$,
$x^T B x  > 0$ and $x^T (B + D) x > 0$; hence
$x^T D x  > - x^T B x$. Thus
$x^T (B + v D) x =  x^T B x + v x^T D x >  (1 - v) x^T B x > 0.$
\end{proof}


\section{Large Deviation Inequalities}
\label{sec:dev-bound}
Before we go on, we explain the notation that we follow throughout this 
section. We switch notation from $t$ to $x$ and form a regression problem for 
non-iid data. Given an interval of $[0, 1]$, the point of interest is $x_0 = 1$.
We form a design matrix by sampling a set of $n$
$p$-dimensional Gaussian random vectors $Z^t$ at $t= 0, 1/n,$ $2/n, \ldots, 1$, 
where $Z^t \sim N(0,\Sigma_t)$ are independently distributed.
In this section, we index the random vectors $Z$ with $k = 0, 1, \ldots, n$ 
such that $Z_k = Z^t$ for $k = nt$, with corresponding covariance matrix denoted
by $\Sigma_k$. Hence
\begin{equation}
\label{eq:ts-vectors}
Z_k = (Z_{k1}, \ldots, Z_{kp})^T \sim N(0,\Sigma_k), \; \forall k.
\end{equation}
These are independent but not identically distributed.
We will need to generalize the usual inequalities.
In Section~\ref{sec:append-boxcar}, 
via a boxcar kernel function, we use moment generating functions to show that
for $\hat\Sigma = \frac{1}{n}\sum_{k=1}^n Z_k Z_k^T$,
\begin{equation}
\label{eq:cov-dev-bound}
P^n(|\hat\Sigma_{ij} - \Sigma_{ij}(x_0)| > \epsilon) < e^{-c n \epsilon^2}
\end{equation}
where
$P^n = P_1\times \cdots \times P_n$ denotes the product measure.
We look across $n$ time-varying Gaussian vectors, and roughly, 
we compare $\hat\Sigma_{ij}$ with $\Sigma_{ij}(x_0)$,
where $\Sigma(x_0) = \Sigma_{n}$ is the covariance matrix in the end of 
the window for $t_0 = n$.
Furthermore, we derive inequalities in Section~\ref{sec:kernel-dev} 
for a general kernel function.

\subsection{Bounds For Kernel Smoothing}

\label{sec:kernel-dev}
In this section, we derive large deviation inequalities for the covariance 
matrix based on kernel regression estimations. Recall that we assume that the 
symmetric nonnegative kernel function $K$ has a bounded support $[-1, 1]$
in A~\ref{as::kernel}.
This kernel has the property that:
\begin{eqnarray}
\label{eq::Ker1}
2 \int_{-1}^{0} v K(v)dv & \leq & 2 \int_{-1}^{0} K(v)dv = 1  \\
\label{eq::Ker2}
2 \int_{-1}^{0} v^2 K(v)dv & \leq & 1.
\end{eqnarray}
In order to estimate $t_0$, instead of taking an average of sample 
variances/covariances over the last $n$ samples, we use the weighting scheme 
such that data close to $t_0$ receives larger weights than those that are far away.
Let $\Sigma(x) = \left(\sigma_{ij}(x)\right)$.
Let us define $x_0 = \frac{t_0}{n} = 1$, and $\forall i = 1, \ldots, n$,
$x_i =  \frac{t_0 - i}{n}$ and
\begin{eqnarray}
\label{eq::kernel-weight}
\ell_i(x_0)  =  \frac{2}{nh} K\left(\frac{x_i - x_0}{h}\right)
\approx \frac{K\left(\frac{x_i - x_0}{h}\right)}{\sum_{i=1}^n K\left(\frac{x_i - x_0}{h}\right)}
\end{eqnarray}
where the approximation is due to replacing the sum with the Riemann integral:
\begin{eqnarray*}
\sum_{i=1}^n \ell_i(x_0) = 
\sum_{i=1}^n \frac{2}{nh} K\left(\frac{x_i - x_0}{h}\right)
\approx 
2 \int_{-1}^{0} K(v) dv = 1,
\end{eqnarray*}
due to the fact that $K(v)$ has compact support in $[-1, 1]$ and $h \leq 1$. Let 
$\Sigma_k = \left(\sigma_{ij}(x_k)\right), \forall k = 1, \ldots, n,$
where 
$\sigma_{ij}(x_k) = \cov(Z_{ki}, Z_{kj}) = \rho_{ij}(x_k)\sigma_{i}(x_k) \sigma_{j}(x_k)$ and
$\rho_{ij}(x_k)$ is the correlation coefficient between $Z_i$ and $Z_j$ at time
$x_k$.
Recall that we have independent $(Z_{ki}Z_{kj})$ for all $k =1, \ldots, n$ 
such that $\expct{(Z_{ki} Z_{kj})} = \sigma_{ij}(x_k)$. Let 
$$\Phi_1(i, j) =
\inv{n}\sum_{k=1}^n \frac{2}{h} K\left(\frac{x_k - x_0}{h}\right) 
\sigma_{ij}(x_k), \; \; \mbox{hence}$$
$$\expct{\sum_{k=1}^n \ell_k(x_0) Z_{ki}Z_{kj}}
=\sum_{k=1}^n \ell_k(x_0) \sigma_{ij}(x_k)  = \Phi_1(i,j).$$
We thus decompose and bound for point of interest $x_0$
\begin{eqnarray}\nonumber
\lefteqn{\size{\sum_{k=1}^n \ell_k(x_0) Z_{ki}Z_{kj} 
- \sigma_{ij}(x_0)} \leq} \\ \nonumber
\label{eq::bd-decompose}
& & \size{\expct{\sum_{k=1}^n \ell_k(x_0)Z_{ki}Z_{kj}} 
- \sigma_{ij}(x_0)}  + \\
& & \size{\sum_{k=1}^n \ell_k(x_0) Z_{ki}Z_{kj}-
 \expct{\sum_{k=1}^n \ell_k(x_0) Z_{ki}Z_{kj}}} \\ \nonumber
& =& \size{\sum_{k=1}^n \ell_k(x_0) Z_{ki}Z_{kj} - \Phi_1(i,j)}  + 
\size{\Phi_1(i,j) - \sigma_{ij}(x_0)}.
\end{eqnarray}
Before we start our analysis on large deviations, we first look at the bias term.
\begin{lemma}\label{lemma::bias}
Suppose there exists $C>0$ such that
$$\max_{i,j}\sup_t |\sigma''(t, i, j)| \leq C. \; \; \mbox{Then}$$
$$
\forall t \in [0, 1], \;\; \max_{i,j}|\expct{\hat{S}_n(t, i, j)} - \sigma_{ij}(t)| = O(h).
$$
\end{lemma}
\begin{proof}
W.l.o.g, let $t = t_0$, hence
$\expct{\hat{S}_n(t, i, j)} = \Phi_1(i,j)$. We use the Riemann integral 
to approximate the sum,
\begin{eqnarray*}
\lefteqn{\Phi_1(i,j) = \inv{n} \sum_{k=1}^n
\frac{2}{h} K\left(\frac{x_k - x_0}{h}\right) \sigma_{ij}(x_k)} \\
& \approx &
\int_{x_n}^{x_0} \frac{2}{h} K\left(\frac{u - x_0}{h}\right) \sigma_{ij}(u) du \\
& = &
2\int_{-1/h}^{0} K(v) \sigma_{ij}(x_0+hv) dv.
\end{eqnarray*}
We now use Taylor's Formula to replace
$\sigma_{ij}(x_0 + hv)$ and obtain
$
2\int_{-1/h}^{0} K(v) \sigma_{ij}(x_0+hv) dv
= \\
2\int_{-1}^{0} K(v)\left(\sigma_{ij}(x_0) +
hv \sigma'_{ij}(x_0) +
\frac{\sigma''_{ij}(y(v))(hv)^2}{2}\right)dv \\
= \sigma_{ij}(x_0) +
2\int_{-1}^{0} K(v)\left(hv \sigma'_{ij}(x_0) +
\frac{C(hv)^2}{2}\right) dv,$
\begin{eqnarray*}
&\mbox{where }&2\int_{-1}^{0} K(v)\left(hv \sigma'_{ij}(x_0) +
\frac{C(hv)^2}{2}\right) dv \\
&= &2 h \sigma'_{ij}(x_0) \int_{-1}^{0} v K(v)dv 
+\frac{Ch^2}{2}  \int_{-1}^{0} v^2 K(v) dv\\
& \leq & h \sigma'_{ij}(x_0) + \frac{C h^2}{4}, \mbox { where } y(v) - x_0 < hv.
\end{eqnarray*}
Thus $\Phi_1(i,j) - \sigma_{ij}(x_0) = O(h)$.
\end{proof}

We now move on to the large deviation bound for all entries of the 
smoothed empirical covariance matrix.
\begin{lemma}
\label{lemma:deviation}
For $\epsilon < 
\frac{C_1 \left(\sigma_{i}^2(x_0) \sigma_{j}^2(x_0)+ \sigma^2_{ij}(x_0)\right)}
{\max_{k = 1, \ldots, n}\left(2K\left(\frac{x_k- x_0}{h}\right) \sigma_{i}(x_k)  
\sigma_{j}(x_k) \right)}$, where $C_1$ is defined in Claim~\ref{claim:Phi-2-bound},
for some $C>0$,
$$
\prob{|\hat{S}_n(t, i, j) - \expct{\hat{S}_n(t, i, j)}| > \epsilon}
\leq \exp\left\{ - C n h \epsilon^2 \right\}.
$$
\end{lemma}
\begin{proof}
Let us define $A_k = Z_{ki} Z_{kj} - \sigma_{ij}(x_k)$.
\begin{eqnarray*}
\lefteqn{\prob{|\hat{S}_n(t,i, j) - \expct{\hat{S}_n(t, i,j)}| > \epsilon}} 
\\ \nonumber
& = & \prob{\sum_{k=1}^n \ell_k(x_0) Z_{ki} Z_{kj} - \sum_{k=1}^n
\ell_k(x_0) \sigma_{ij}(x_k) > \epsilon} 
\end{eqnarray*}
For every $t > 0$, we have by Markov's inequality
\begin{eqnarray}
\label{eq:kernel-markov}
\nonumber
\lefteqn{\prob{\sum_{k=1}^n n \ell_k(x_0) A_k  > n \epsilon}} \\ \nonumber
 & = & 
\prob{e^{t\sum_{k=1}^n \frac{2}{h} K\left(\frac{x_i - x_0}{h}\right) A_k}  > e^{nt\epsilon}} \\
& \leq &
\frac{\expct{e^{t\sum_{k=1}^n \frac{2}{h} K\left(\frac{x_i - x_0}{h}\right)  A_k}}}
{e^{nt\epsilon}}.
\end{eqnarray}
Before we continue, for a given $t$, let  us first define the following quantities,
where $i,j$ are omitted from $\Phi_1(i, j)$
\begin{itemize}
\item
$a_k = \frac{2t}{h} K\left(\frac{x_k - x_0}{h}\right)
(\sigma_{i}(x_k) \sigma_{j}(x_k)  + \sigma_{ij}(x_k))$ 
\item
$b_k = \frac{2t}{h} K\left(\frac{x_k - x_0}{h}\right)
 (\sigma_{i}(x_k) \sigma_{j}(x_k) - \sigma_{ij}(x_k))$
\mbox{ thus }
\item
$\Phi_1 = \inv{n}\sum_{k=1}^n \frac{a_k - b_k}{2t},$ \;
$\Phi_2 = \inv{n}\sum_{k = 1}^n \frac{a_k^2 + b_k^2}{4t^2}$
\item
$\Phi_3 = \inv{n}\sum_{k = 1}^n \frac{a_k^3 - b_k^3}{6t^3},$ \; 
$\Phi_4 = \inv{n}\sum_{k = 1}^n \frac{a_k^4 + b_k^4}{8t^4}$
\item
$M = \max_{k=1, \ldots, n}\left(\frac{2}{h} K\left(\frac{x_k- x_0}{h}\right) \sigma_{i}(x_k)  \sigma_{j}(x_k) \right)$
\end{itemize}
We now establish some convenient comparisons; see 
Section~\ref{sec:append-kernel-dev-1} and \ref{sec:append-taylor-sums} for
their proofs.
\begin{claim}
\label{claim:phi-bounds}
$\frac{\Phi_3}{\Phi_2} \leq \frac{4 M}{3}$ and 
$\frac{\Phi_4}{\Phi_2} \leq 2 M^2$, where both equalities are established 
at $\rho_{ij}(x_k)  = 1, \forall k$.
\end{claim}
\begin{lemma}
\label{lemma:taylor-sums}
For $b_k \leq a_k \leq \half, \forall k$,
$\half\sum_{k=1}^n \ln \inv{(1 - a_k)(1+b_k)}$
$\leq  nt \Phi_1 + nt^2 \Phi_2 + nt^3 \Phi_3 + \frac{9}{5} nt^4 \Phi_4$.
\end{lemma}
To show the following, 
we first replace the sum with a Riemann integral, and then use Taylor's Formula
to approximate $\sigma_{i}(x_k), \sigma_{j}(x_k)$, and
$\sigma_{ij}(x_k), \forall k = 1, \ldots, n$ with $\sigma_{i}, \sigma_{j}$
$\sigma_{ij}$ and their first derivatives at $x_0$ respectively, plus some 
remainder terms; see Section~\ref{sec:append-kernel-dev-2} for details.
\begin{claim}
\label{claim:Phi-2-bound}
For $h = n^{-\epsilon}$ for some $1> \epsilon > 0$, there exists some constant 
$C_1 > 0$ such that
\begin{eqnarray*}
\Phi_2(i,j) = 
\frac{C_1 (\sigma_{i}^2(x_0) \sigma_{j}^2(x_0) + \sigma^2_{ij}(x_0))}{h}.
\end{eqnarray*}
\end{claim}
Lemma\ref{lemma:moment2} computes the moment generating function for 
$\frac{2}{h} K\left(\frac{x_k - x_0}{h}\right) Z_{ki}\cdot Z_{kj}$.
The proof proceeds exactly as that of Lemma~\ref{lemma:moment1} 
after substituting $t$ with 
$\frac{2t}{h} K\left(\frac{x_k - x_0}{h}\right)$ everywhere.
\begin{lemma}
\label{lemma:moment2}
Let
$\frac{2t}{h} K\left(\frac{x_k - x_0}{h}\right)(1 + \rho_{ij}(x_k)) \sigma_{i}(x_k)  \sigma_{j}(x_k)$ 
$< 1, \forall k$. For $b_k \leq a_k < 1$.
\begin{eqnarray}
\nonumber
\expct{e^{\frac{2t}{h} K\left(\frac{x_k - x_0}{h}\right) Z_{ki} Z_{kj}}}
 = \left((1- a_k)(1+ b_k)\right)^{-1/2}.
\end{eqnarray}
\end{lemma}
\begin{remark}
\label{remark:bound-on-ab}
Thus when we set $t = \frac{\epsilon}{4 \Phi_2}$, the bound on $\epsilon$
implies that  $b_k \leq a_k \leq 1/2, \forall k$:
\begin{eqnarray*}
\label{eq:ak-bound}
a_k &= &t(1 + \rho_{ij}(x_k)) \sigma_{i}(x_k) \sigma_{j}(x_k) \\
& \leq & 2t \sigma_{i}(x_k) \sigma_{j}(x_k)  
= \frac{\epsilon \sigma_{i}(x_k) \sigma_{j}(x_k)}{2 \Phi_2} \leq \half.
\end{eqnarray*}
\end{remark}
We can now finish showing the large deviation bound for 
$\max_{i, j}|\hat{S}_{i, j} - \expct{S_{i, j}}|$.
Given that $A_1, \ldots, A_n$ are independent, we have 
\begin{eqnarray}
\label{eq:kernel-indep-moment}
\nonumber 
\lefteqn{\expct{e^{t\sum_{k=1}^n \frac{2}{h} K\left(\frac{x_k - x_0}{h}\right) A_k}}
= \prod_{k=1}^n
\expct{e^{\frac{2t}{h} K\left(\frac{x_1 - x_0}{h}\right) A_k}}} \\ \nonumber
& = & \prod_{k=1}^n 
\exp\left({-\frac{2t}{h} K\left(\frac{x_k - x_0}{h}\right) \sigma_{ij}(x_k)}\right)\cdot\\
& &  \prod_{k=1}^n \expct{e^{\frac{2t}{h} K\left(\frac{x_k - x_0}{h}\right) 
Z_{ki} Z_{kj}}}
\end{eqnarray}
By~(\ref{eq:kernel-markov}),~(\ref{eq:kernel-indep-moment}),
Lemma~\ref{lemma:moment2}, for $t \leq \frac{\epsilon}{4\Phi_2}$,
\begin{eqnarray*}
\label{eq:kernel-dev-bound}
\nonumber
\lefteqn{
\prob{\sum_{k=1}^n \frac{2}{h} K\left(\frac{x_k - x_0}{h}\right) A_k > n\epsilon}} \\ 
&\leq & 
\frac{\expct{e^{t\sum_{k=1}^n \frac{2}{h} K\left(\frac{x_k - x_0}{h}\right)A_k}}}{e^{-nt\epsilon}}  =  e^{-nt\epsilon} \cdot \\
& \prod_{k=1}^n & e^{-\frac{2t}{h} K\left(\frac{x_k - x_0}{h}\right) \sigma_{ij}(x_k)}
\cdot\expct{e^{\frac{2t}{h} K\left(\frac{x_k - x_0}{h}\right) 
Z_{ki} Z_{kj}}} \\
& = & 
e^{-nt\epsilon - nt\Phi_1(i,j)
+ \half\sum_{k=1}^n \ln\inv{(1 - a_k)(1+b_k)}} \\
& \leq &
\exp\left(-nt\epsilon +  n t^2 \Phi_2 +  n t^3 \Phi_3 +  \frac{9}{5} n t^4 \Phi_4 \right),
\end{eqnarray*}
where the last step is due to Remark~\ref{remark:bound-on-ab} and
Lemma~\ref{lemma:taylor-sums}.
Now let us consider taking $t$ that minimizes

\noindent
$\exp\left(- nt\epsilon +  n t^2\Phi_2 +  n t^3 \Phi_3 +  \frac{9}{5} n t^4 \Phi_4 \right)$; Let $t = \frac{\epsilon}{4\Phi_2}$: \\
$\frac{d}{dt}
\left(-nt\epsilon +  n t^2 \Phi_2 +  n t^3 \Phi_3 +  \frac{9}{5} n t^4 \Phi_4 \right)
\leq -\frac{\epsilon}{40};$
Now given that
$\frac{\epsilon^2}{\Phi_2} < \inv{M}$, Claim~\ref{claim:phi-bounds} and~\ref{claim:Phi-2-bound}:
\begin{eqnarray*}
\label{eq:kernel-dev-bound-cont}
& & 
\prob{\sum_{k=1}^n \frac{2}{h} K\left(\frac{x_k - x_0}{h}\right) A_k > n\epsilon} \\
& {\hskip-.38in} \leq & {\hskip-.23in}
\exp\left(-nt\epsilon +  n t^2 \Phi_2 +  nt^3 \Phi_3 +  \frac{9}{5}nt^4  \Phi_4 \right) \\ \nonumber
& {\hskip-.38in} \leq & {\hskip-.23in}
\exp\left(\frac{-n\epsilon^2}{4\Phi_2} + \frac{n \epsilon^2}{16 \Phi_2} + 
 \frac{n\epsilon^2}{64\Phi_2}\frac{\epsilon\Phi_3}{\Phi_2^2} +  
\frac{9}{5} \frac{n\epsilon^2}{256\Phi_2}\frac{\epsilon^2 \Phi_4}{\Phi_2^3} \right) \\
& {\hskip-.38in} \leq & {\hskip-.23in}
\exp\left(\frac{-3 n\epsilon^2}{20 \Phi_2}\right) \\
& {\hskip-.38in} \leq & {\hskip-.23in}
\exp\left(- \frac{3 n h\epsilon^2}{20 C_1 ( \sigma_i^2(x_0) \sigma_j^2(x_0) 
+\sigma_{ij}^2(x_0))}\right).
\end{eqnarray*}
Finally, let's check the requirement on $\epsilon \leq \frac{\Phi_2}{M}$,
\begin{eqnarray*}
\epsilon & \leq &
\frac{\left(
C_1 (1+ \rho^2_{ij}(x_0)) \sigma_{i}^2(x_0) \sigma_{j}^2(x_0)\right)/h}
{\max_{k = 1, \ldots, n}\left(\frac{2}{h} K\left(\frac{x_k- x_0}{h}\right) \sigma_{i}(x_k)  \sigma_{j}(x_k) \right)} \\
& = &
\frac{\left(
C_1 (1+ \rho^2_{ij}(x_0)) \sigma_{i}^2(x_0) \sigma_{j}^2(x_0)\right)}
{\max_{k = 1, \ldots, n}\left(2K\left(\frac{x_k- x_0}{h}\right) \sigma_{i}(x_k)  
\sigma_{j}(x_k) \right)}.
\end{eqnarray*}
\end{proof}

For completeness,
we compute the moment generating function for $Z_{k, i} Z_{k, j}$.
\begin{lemma}
\label{lemma:moment1}
Let $t (1 + \rho_{ij}(x_k)) \sigma_{i}(x_k) \sigma_{j}(x_k) < 1, \forall k$, so
that $b_k \leq a_k < 1$, omitting $x_k$ everywhere,
\begin{eqnarray*}
\lefteqn{\expct{e^{t Z_{k, i} Z_{k, j}}} = } \\
& & {\hskip-1.2cm}\left(\inv{(1- t(\sigma_{i} \sigma_{j} + \sigma_{ij}) 
(1+ t(\sigma_{i}\sigma_{j} - \sigma_{ij})) }\right)^{1/2}.
\end{eqnarray*}
\end{lemma}
\begin{proof}
W.l.o.g., let $i=1$ and $j=2$.
\begin{eqnarray*}
\lefteqn{\expct{\left(e^{t Z_1 Z_2}\right)}
 = \expct{\left(\expct{\left(e^{tZ_2 Z_1}|Z_2\right)}\right)}} \\
& = & 
\expct{\exp\left(\left(
\frac{t \rho_{12}\sigma_1}{\sigma_2} + \frac{t^2 \sigma_1^2(1-\rho_{12}^2)}{2}\right)Z_2^2\right)} \\
& = & \left(1-2 \left(\frac{t \rho_{12} \sigma_1}{\sigma_2} + \frac{t^2 \sigma_1^2(1-\rho_{12}^2)}{2}\right) \sigma_2^2\right)^{-1/2} \\
& = & \left(\inv{1-\left(2t \rho_{12}\sigma_1\sigma_2 + t^2 \sigma_1^2 \sigma_2^2(1-\rho_{12}^2)\right)}\right)^{1/2} \\
& = & \left(\inv{(1- t(1 + \rho_{12}) \sigma_1 \sigma_2)
(1+ t(1 - \rho_{12}) \sigma_{1} \sigma_{2})}\right)^{1/2}
\end{eqnarray*}
where $2t \rho_{12}\sigma_1\sigma_2 + t^2 \sigma_1^2 \sigma_2^2(1-\rho_{12}^2) < 1$.
This requires that $t < \inv{(1+ \rho_{12}) \sigma_{1} \sigma_{2}}$ which
is equivalent to 
$2t \rho_{12}\sigma_1\sigma_2 + t^2 \sigma_1^2 \sigma_2^2(1-\rho_{12}^2) - 1 
< 0$. One can check that if we require
$t(1+\rho_{12})\sigma_1\sigma_2 \leq  1$, which implies 
that $t \sigma_1 \sigma_2 \leq 1 - t \rho_{12}\sigma_1\sigma_2$ and hence 
$t^2 \sigma_1^2 \sigma_2^2 \leq (1 - t \rho_{12}\sigma_1\sigma_2)^2$, the lemma holds.
\end{proof}

\section{Smoothness and Sparsity of $\Sigma_t$ via $\Sigma_t^{-1}$}
\label{sec:sparsity}
In this section we show that if we assume 
$\Theta(x) = \left( \theta_{ij}(x)\right)$ are smooth and twice
differentiable functions of $x \in [0, 1]$, i.e.,
$\theta'_{ij}(x) < \infty$ and 
$\theta''_{i j}(x) < \infty$ for $x \in [0,1], \forall i, j$, and satisfy
$A~\ref{as::Theta}$, then the smoothness conditions of
Lemma~\ref{lemma:common} are satisfied.
The following is a standard result in matrix analysis.
\begin{lemma}
\label{lemma:Sigma-first-derivitive}
Let $\Theta(t) \in R^{p \times p}$ has entries that are differentiable functions 
of $t \in [0, 1]$. Assuming that $\Theta(t)$ is always non-singular, then
$$
\frac{d}{dt}[\Sigma(t)] = - \Sigma(t) \frac{d}{dt}[\Theta(t)] \Sigma(t).
$$ 
\end{lemma}
\silent{
We omit the proof of the above lemma, which is standard in matrix analysis;
instead, we just point out that since each entry in $\Sigma(t)$ can be expressed 
as $f/g$ such that functions $f$ and $g$ are the determinants of a submatrix 
of $\Theta(t)$ and $\Theta(t)$ itself respectively, both of which are polynomials of 
entries of $\theta_{ij}(t)$, which are differentiable functions for all $i, j$.
Hence $\frac{d}{dt}[\Sigma(t)]$ exists so long as $\Theta(t)$ is non-singular.}
\begin{lemma}
\label{lemma:second-derivitive}
Suppose $\Theta(t) \in R^{p \times p}$ has entries that each are twice differentiable
functions of $t$. Assuming that $\Theta(t)$ is always non-singular, then
$$
\frac{d^2}{dt^2}\left[\Sigma(t)\right] = \Sigma(t) D(t) \Sigma(t), \; \; \mbox {where}
$$
$$D(t) = 2 \frac{d}{dt}[\Theta(t)] \Sigma(t)\frac{d}{dt}[\Theta(t)]
- \frac{d^2}{dt^2}[\Theta(t)].$$
\end{lemma}
\begin{proof}
The existence of the second order derivatives for entries of $\Sigma(t)$ is due to 
the fact that $\Sigma(t)$ and $\frac{d}{dt}[\Theta(t)]$ are both differentiable
$\forall t \in [0, 1]$; indeed by Lemma~\ref{lemma:Sigma-first-derivitive}, 
\begin{eqnarray*} \nonumber
& &\frac{d^2}{dt^2}[\Sigma(t)] 
 = 
\frac{d}{dt}\left[- \Sigma(t) \frac{d}{dt}[\Theta(t)] \Sigma(t)\right] \\\nonumber
& = & 
- \frac{d}{dt}[\Sigma(t)] \frac{d}{dt}[\Theta(t)] \Sigma(t)
- \Sigma(t) \frac{d}{dt}\left[\frac{d}{dt}[\Theta(t)] \Sigma(t)\right] \\\nonumber
\label{eq:cov-2-def}
& = & 
- \frac{d}{dt}[\Sigma(t)] \frac{d}{dt}[\Theta(t)] \Sigma(t) 
- \Sigma(t) \frac{d^2}{dt^2}[\Theta(t)]\Sigma(t)
- \\ 
&& \Sigma(t) \frac{d}{dt}[\Theta(t)]  \frac{d}{dt}[\Sigma(t)] \\ \nonumber
& = & 
\Sigma(t) \left(2 \frac{d}{dt}[\Theta(t)] \Sigma(t) \frac{d}{dt}[\Theta(t)]
- \frac{d^2}{dt^2}[\Theta(t)] \right)\Sigma(t),
\end{eqnarray*}
hence the lemma holds by the definition of $D(t)$.
\end{proof}

Let $\Sigma(x) = \left(\sigma_{ij}(x)\right), \forall x \in [0, 1]$.
Let $\Sigma(x) = (\Sigma_1(x), \Sigma_2(x), \ldots, \Sigma_p(x))$, where 
$\Sigma_i(x) \in R^p$ denotes a column vector. 
By Lemma~\ref{lemma:second-derivitive},
\begin{eqnarray}
\label{eq:sigma'}
\sigma'_{ij}(x)& = & - \Sigma^T_i(x) \Theta'(x) \Sigma_j(x), \\
\label{eq:sigma''}
\sigma''_{ij}(x) & = & \Sigma^T_i(x) D(x) \Sigma_j(x),
\end{eqnarray}
where 
$\Theta'(x) = \left(\theta'_{ij}(x)\right), \forall x \in [0, 1].$
\begin{lemma}
\label{lemma:sigma-first-derivitive}
Given A~\ref{as::S_0} and A~\ref{as::Theta},  $\forall x \in [0, 1]$, 
$$|\sigma'_{ij}(x)| \leq  S_0^2 \sqrt{S_1} < \infty.$$
\end{lemma}
\begin{proof}
$|\sigma'_{ij}(x)| = |\Sigma^T_i(x) \Theta'(x) \Sigma_j(x)|$ \\
$$\leq \max_{i=1 \ldots, p}|\sigma^2_{i}(x)|
\sum_{k=1}^p \sum_{\ell=1}^p \abs{\theta'_{k\ell}(x)} \leq S_0^2 \sqrt{S_1}.$$
\end{proof}

We denote the elements of $\Theta(x)$ by $\theta_{jk}(x)$.
Let $\theta'_{\ell}$ represent a column vector of $\Theta'$.
\begin{theorem}
\label{thm:smoothness}
Given A~\ref{as::S_0} and A~\ref{as::Theta}, $\forall i, j$, $\forall x \in [0, 1]$,
$$\sup_{x \in[0, 1]}\abs{\sigma''_{ij}(x)} < 2 S_0^3 S_1 + S_0^2 S_2 < \infty.$$
\end{theorem}
\begin{proof}
By~(\ref{eq:sigma''}) and the triangle inequality,
\begin{eqnarray*}
\lefteqn{\abs{\sigma''_{ij}(x)} =   \abs{\Sigma^T_i(x) D(x) \Sigma_j(x)}} \\
& \leq & 
\max_{i=1 \ldots, p}\abs{\sigma^2_{i}(x)} 
\sum_{k=1}^p \sum_{\ell=1}^p |D_{k\ell}(x)| \\
& \leq & 
\label{eq:two-terms-L}
S_0^2 
\sum_{k=1}^p \sum_{\ell=1}^p 2|\theta'^T_{k}(x) \Sigma(x) \theta'_{\ell}(x)|
+
|\theta''_{k \ell}(x)|\\
& = & 
\label{eq:two-terms-R}
2 S_0^3 S_1 + S_0^2 S_2,
\end{eqnarray*}
where by A~\ref{as::Theta}, 
$\sum_{k=1}^p \sum_{\ell=1}^p \abs{\theta''_{k \ell}(x)}
\leq S_2$, and
\begin{eqnarray*}
\lefteqn{\sum_{k=1}^p \sum_{\ell=1}^p \abs{\theta'^T_{k}(x) \Sigma(x) \theta'_{\ell}(x)}} \\
 & = &
\sum_{k=1}^p \sum_{\ell=1}^p \sum_{i=1}^p \sum_{j=1}^p 
\abs{\theta'_{k i}(x) \theta'_{\ell j}(x) \sigma_{ij}(x) } \\
& \leq & 
\max_{i=1 \ldots, p}\abs{\sigma_{i}(x)}
\sum_{k=1}^p \sum_{\ell=1}^p \sum_{i=1}^p \sum_{j=1}^p 
\abs{\theta'_{k i}(x) \theta'_{\ell j}(x)}
\end{eqnarray*} 
$\hspace{1cm} \leq S_0 S_1$.
\end{proof}

\section{Some Implications of a Very Sparse $\Theta$}
\label{sec:append-implications}
We use $\Lebesgue^1$ to denote Lebesgue measure on $\R$. 
The aim of this section is to prove some bounds that correspond to 
$A~\ref{as::Theta}$, but only for $\Lebesgue^1$ a.e. $x \in [0, 1]$, based on
a single sparsity assumption on $\Theta$ as in $A$~\ref{as::sparsity}.
We let $E \subset [0, 1]$ represent the ``bad'' set with $\Lebesgue^1(E) = 0$.
and $\Lebesgue^1$ a.e. $x \in [0, 1]$ refer to points in the set 
$[0, 1]\setminus E$ such that $\Lebesgue^1([0, 1] \setminus E) = 1$.
When $\norm{\Theta(x)}_0 \leq s + p$ for all $x \in [0, 1]$, we immediately obtain
Theorem~\ref{thm:sparsity}, 
whose proof appears in Section~\ref{sec:app-thm-sparsity}.
We like to point out that although we apply Theorem~\ref{thm:sparsity} to 
$\Theta$ and deduce smoothness of $\Sigma$, 
we could apply it the other way around. In particular,
it might be interesting to apply it to the correlation coefficient 
matrix $(\rho_{ij})$, where the diagonal entries remain invariant.
We use $\Theta'(x)$ and $\Theta''(x)$ to denote $(\theta_{ij}'(x))$ and 
$(\theta_{ij}''(x))$ respectively $\forall x$.
\begin{assumption}
\label{as::sparsity}
Assume that $\norm{\Theta(x)}_0 \leq s + p$ $\forall x \in [0, 1]$.
\end{assumption}
\begin{assumption}
\label{as::theta-der-bounds}
$\exists S_4, S_5 < \infty$ such that
\begin{gather}
S_4 = \max_{ij} \norm{\theta'_{i j}}_{\infty}^2 \; \; \mbox{and} \; \;
S_5 = \max_{ij} \norm{\theta''_{i j}}_{\infty}.
\end{gather}
\end{assumption}
We state a theorem, the proof of which is in Section~\ref{sec:app-thm-sparsity} and
a corollary.
\begin{theorem}
\label{thm:sparsity}
Under $A$~\ref{as::sparsity}, we have
$\norm{\Theta''(x)}_0 \leq$ $\norm{\Theta'(x)}_0$ 
$\leq \norm{\Theta(x)}_0 \leq s + p$ for $\Lebesgue^1$ a.e. $x \in [0, 1]$. 
\end{theorem}
\begin{corollary}
Given A~\ref{as::S_0} and A~\ref{as::sparsity}, 
for $\Lebesgue^1$ a.e. $x \in [0, 1]$
\begin{eqnarray}
|\sigma'_{ij}(x)| \leq S_0^2 \sqrt{S_4} (s+p) < \infty.
\end{eqnarray}
\end{corollary}
\begin{proof}
By proof of Lemma~\ref{lemma:sigma-first-derivitive},

\noindent $|\sigma'_{ij}(x)| \leq 
\max_{i=1 \ldots, p}\|\sigma^2_{i}\|_{\infty}
\sum_{k=1}^p \sum_{\ell=1}^p |\theta'_{k\ell}(x)|$. \\
Hence by Theorem~\ref{thm:sparsity}, for $\Lebesgue^1$ a.e. $x \in [0, 1]$,
$|\sigma'_{ij}(x)| \leq 
\max_{i=1 \ldots, p}\|\sigma^2_{i}\|_{\infty}
\sum_{k=1}^p \sum_{\ell=1}^p |\theta'_{k\ell}(x)| \\
\leq S_0^2 \max_{k, \ell} \norm{\theta'_{k \ell}}_{\infty} \norm{\Theta'(x)}_0
\leq S_0^2 \sqrt{S_4} (s +p).$
\end{proof}
\begin{lemma}
\label{lemma:sparse-item}
Under $A$~\ref{as::sparsity} and~\ref{as::theta-der-bounds},
for $\Lebesgue^1$ a.e. $x \in [0, 1]$,
\begin{eqnarray*}
\label{eq:first-deriv-item} \nonumber
\sum_{k=1}^p \sum_{\ell=1}^p \sum_{i=1}^p \sum_{j=1}^p 
\abs{\theta'_{k i}(x) \theta'_{\ell j}(x)}
\leq (s+p)^2 \max_{ij} \norm{\theta'_{i j}}_{\infty}^2 \\
\label{eq:second-deriv-item}
\sum_{k=1}^p \sum_{\ell=1}^p \theta''_{k \ell}
 \leq (s+p) \max_{ij} \norm{\theta''_{i j}}_{\infty}, \; \; \text{hence} \\
{\rm ess }\sup_{x \in [0, 1]}{\sigma''_{ij}(x)} \leq  
2 S_0^3 (s+p)^2 S_4 + S_0^2 (s+p) S_5.
\end{eqnarray*}
\end{lemma}
\begin{proof}
By the triangle inequality, for $\Lebesgue^1$ a.e. $x \in [0, 1]$, 
\begin{eqnarray*}
\lefteqn{\abs{\sigma''_{ij}(x)}  =   
\abs{\Sigma^T_i D \Sigma_j}} \\
& = &
\abs{\sum_{k=1}^p \sum_{\ell=1}^p \sigma_{ik}(x) \sigma_{j \ell}(x) D_{k \ell}(x)} \\ 
& \leq & 
\max_{i=1 \ldots, p}\norm{\sigma^2_{i}}_{\infty} 
\sum_{k=1}^p \sum_{\ell=1}^p |D_{k\ell}(x)| \\
& \leq & 
2 S_0^2 \sum_{k=1}^p \sum_{\ell=1}^p |\theta'^T_{k} \Sigma \theta'_{\ell}|
+ S_0^2 \sum_{k=1}^p \sum_{\ell=1}^p |\theta''_{k \ell}|\\
& = & 
2 S_0^3 (s+p)^2 S_4 + S_0^2 (s+p) S_5,
\end{eqnarray*}
where for $\Lebesgue^1$ a.e. $x \in [0, 1]$,\\
\begin{eqnarray*}
\lefteqn{\sum_{k=1}^p \sum_{\ell=1}^p \abs{\theta'^T_{k} \Sigma \theta'_{\ell}}
\leq \sum_{k=1}^p \sum_{\ell=1}^p \sum_{i=1}^p \sum_{j=1}^p 
\abs{\theta'_{k i} \theta'_{\ell j} \sigma_{ij} }} \\
&\leq &
\max_{i=1 \ldots, p}\norm{\sigma_{i}}_{\infty}
\sum_{k=1}^p \sum_{\ell=1}^p \sum_{i=1}^p \sum_{j=1}^p
 \abs{\theta'_{k i} \theta'_{\ell j}}\\
& \leq & S_0 (s+p)^2 S_4
\end{eqnarray*}
and
$\sum_{k=1}^p \sum_{\ell=1}^p |\theta''_{k \ell}|\leq (s+p) S_5.$
The first inequality is due to the following observation:
at most $(s+p)^2$ elements in the sum of 
$\sum_{k} \sum_{i} \sum_{\ell} \sum_{j} 
\abs{\theta'_{k i}(x) \theta'_{\ell j}(x)}$
for $\Lebesgue^1$ a.e. $x \in [0, 1]$, that is, 
except for $E$, are non-zero, due to the fact that for $x \in [0, 1] \setminus N$,
$\norm{\Theta'(x)}_0 \leq \norm{\Theta(x)}_0 \leq s+p$ as in 
Theorem~\ref{thm:sparsity}.
The second inequality is obtained similarly using the fact that
for $\Lebesgue^1$ a.e. $x \in [0, 1]$, 
$\norm{\Theta''(x)}_0 \leq \norm{\Theta(x)}_0 \leq s+p$.
\end{proof}
\begin{remark}
For the bad set $E \subset [0, 1]$ with $\Lebesgue^1(E) = 0$,
$\sigma'_{ij}(x)$ is well defined as shown in Lemma~\ref{lemma:Sigma-first-derivitive}, but it can only be loosely bounded by $O(p^2)$, as $\norm{\Theta'(x)}_0 = O(p^2)$, 
instead of $s+p$, for $x \in E$; similarly, $\sigma''_{ij}(x)$ can only be loosely
bounded by $O(p^4)$.
\end{remark}

By Lemma~\ref{lemma:sparse-item}, using the Lebesgue integral, 
we can derive the following corollary.
\begin{corollary}
\label{cor:sparse-smoothness}
Under A~\ref{as::S_0}, A~\ref{as::sparsity}, and A~\ref{as::theta-der-bounds},
\begin{equation*}
\int_0^1 \left(\sigma''_{ij}(x)\right)^2 dx \leq 2S_0^3 S_4 s+p^2 + S_0^2 S_5 (s+p) < \infty.
\end{equation*}
\end{corollary}
\silent{
$\left(\int_0^1 \left(\sum_{k=1}^p \sum_{\ell=1}^p \sum_{i=1}^p \sum_{j=1}^p 
\abs{\theta'_{k i}(x) \theta'_{\ell j}(x)}\right)^2 dx\right)^{\half}
\\ \leq S_4 (s+p)^2$, and
$\left(\int_0^1
\left(\sum_{k=1}^p \sum_{\ell=1}^p \theta''_{k \ell}\right)^2 
dx\right)^{\half} \leq  S_5 (s+p)$.
The Corollary follows by plugging
~(\ref{eq:der-2-term}) and~(\ref{eq:two-terms-L}) in.}
\subsection{Proof of Theorem~\ref{thm:sparsity}.}
\label{sec:app-thm-sparsity}
Let $\norm{\Theta(x)}_0 \leq s+p$ for all $x \in [0, 1]$.
\begin{lemma}
\label{lemma:leb-first}
Let a function $u:[0, 1] \rightarrow \R$. Suppose $u$ has a derivative on $F$ 
(finite or not) with $\Lebesgue^1(u(F)) = 0$. Then 
$u'(x) = 0$ for $\Lebesgue^1$ a.e. $x \in F$.
\end{lemma}
Take $F = \{x \in [0, 1]: \theta_{ij}(x) = 0\}$ and $u = \theta_{ij}$.
For $\Lebesgue^1$ a.e. $x \in F$, that is, except for a set $N_{ij}$ of 
$\Lebesgue^1(N_{ij}) =0$, $\theta_{ij}'(x) = 0$. Let $N = \bigcup_{ij} N_{ij}$.
By Lemma~\ref{lemma:leb-first},
\begin{lemma}
If $x \in [0, 1] \setminus N$, where $\Lebesgue^1(N) = 0$,
if $\theta_{ij}(x) = 0$, then $\theta_{ij}'(x) = 0$ for all $i, j$.
\end{lemma}
Let $v_{ij} = \theta'_{ij}$.
Take $F = \{x \in [0, 1]: v_{ij}(x) = 0\}$. For $\Lebesgue^1$ a.e. $x \in F$,
that is, except for a set $N^1_{ij}$ with $\Lebesgue(N^1_{ij}) = 0$, 
$v'_{ij}(x) = 0$. Let $N_1 = \bigcup_{ij} N^1_{ij}$. By Lemma~\ref{lemma:leb-first},
\begin{lemma}
If $x \in [0, 1] \setminus N_1$, where $\Lebesgue^1(N_1) = 0$,
if $\theta'_{ij}(x) = 0$, then $\theta_{ij}''(x) = 0, \forall i, j$.
\end{lemma}
Thus this allows to conclude that
\begin{lemma}
If $x \in [0, 1] \setminus N \cup N_1$, where $\Lebesgue^1(N \cup N_1) = 0$,
if $\theta_{ij}(x) = 0$, then  $\theta_{ij}'(x) = 0$ and $\theta_{ij}''(x) = 0, \forall i, j$.
\end{lemma}
Thus for all $x \in [0, 1] \setminus N \cup N_1$,
$\norm{\Theta''(x)}_0 \leq  \norm{\Theta'(x)}_0 \leq \norm{\Theta(x)}_0 \leq (s+p)$.
\; \; \; $\square$

\def\sleft{\hskip-5pt}
\def\lleft{\hskip-25pt}
\section{Examples}
In this section, we demonstrate the effectiveness of the method in a
simulation. Starting at time $t=t_0$, the original graph is as shown
at the top of Figure~1.
The graph evolves according to a type of Erd\H{o}s-R\'{e}nyi random graph 
model. Initially we set $\Theta = 0.25 I_{p \times p}$, where $p = 50$.
Then, we randomly select $50$ edges and update $\Theta$ as follows:
for each new edge $(i, j)$, a weight $a >0$ is chosen
uniformly at random from $[0.1, 0.3]$; we subtract $a$ from 
$\theta_{ij}$ and $\theta_{ji}$, and increase $\theta_{ii}, \theta_{jj}$ 
by $a$. This keeps $\Sigma$ positive definite.
\begin{figure}[htbp]
\begin{center}
\begin{tabular}{c}
\includegraphics[width=.30\textwidth,angle=-0]{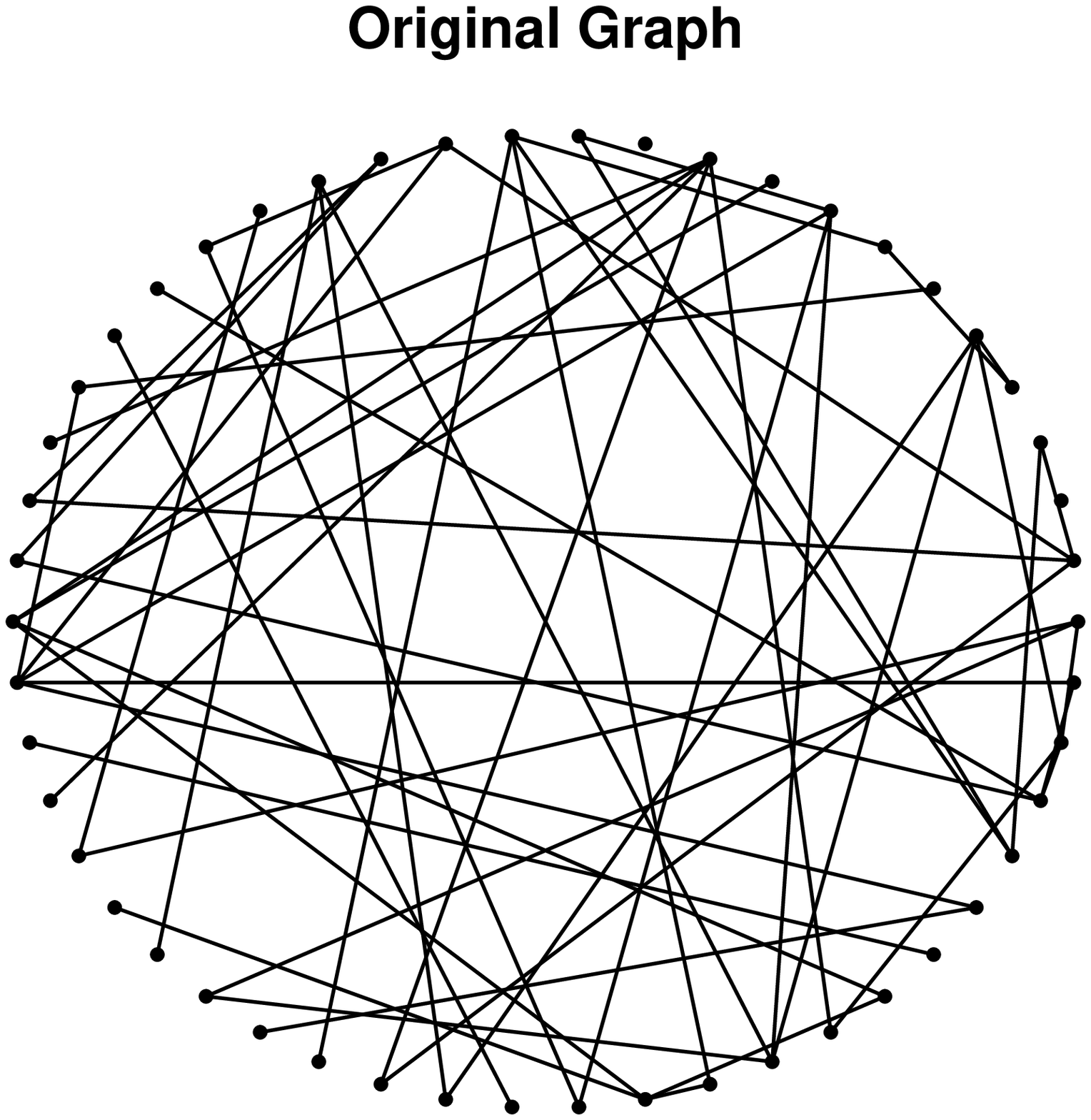} \\[-20pt]
\includegraphics[width=.30\textwidth,angle=-90]{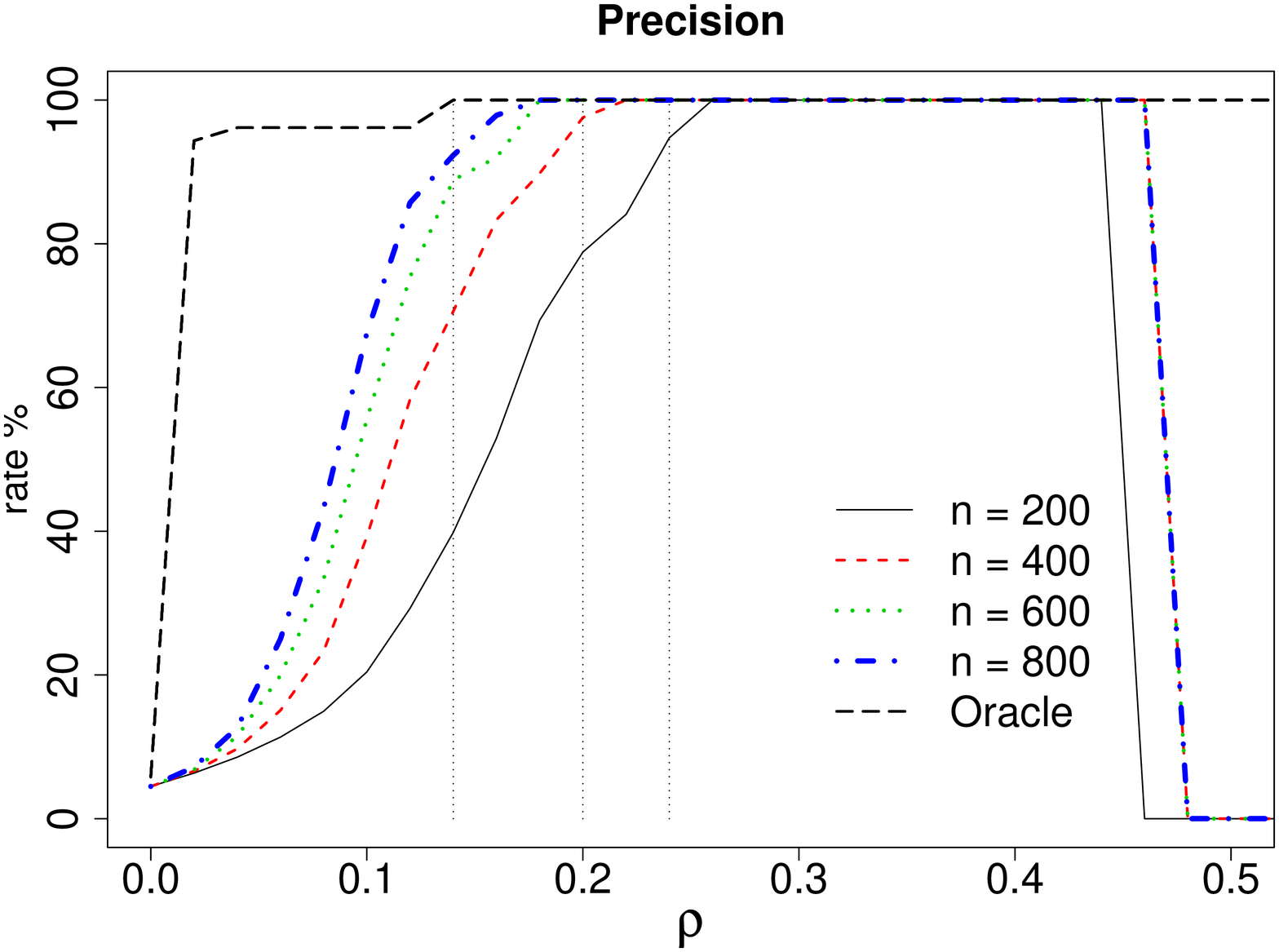} \\[-10pt]
\includegraphics[width=.30\textwidth,angle=-90]{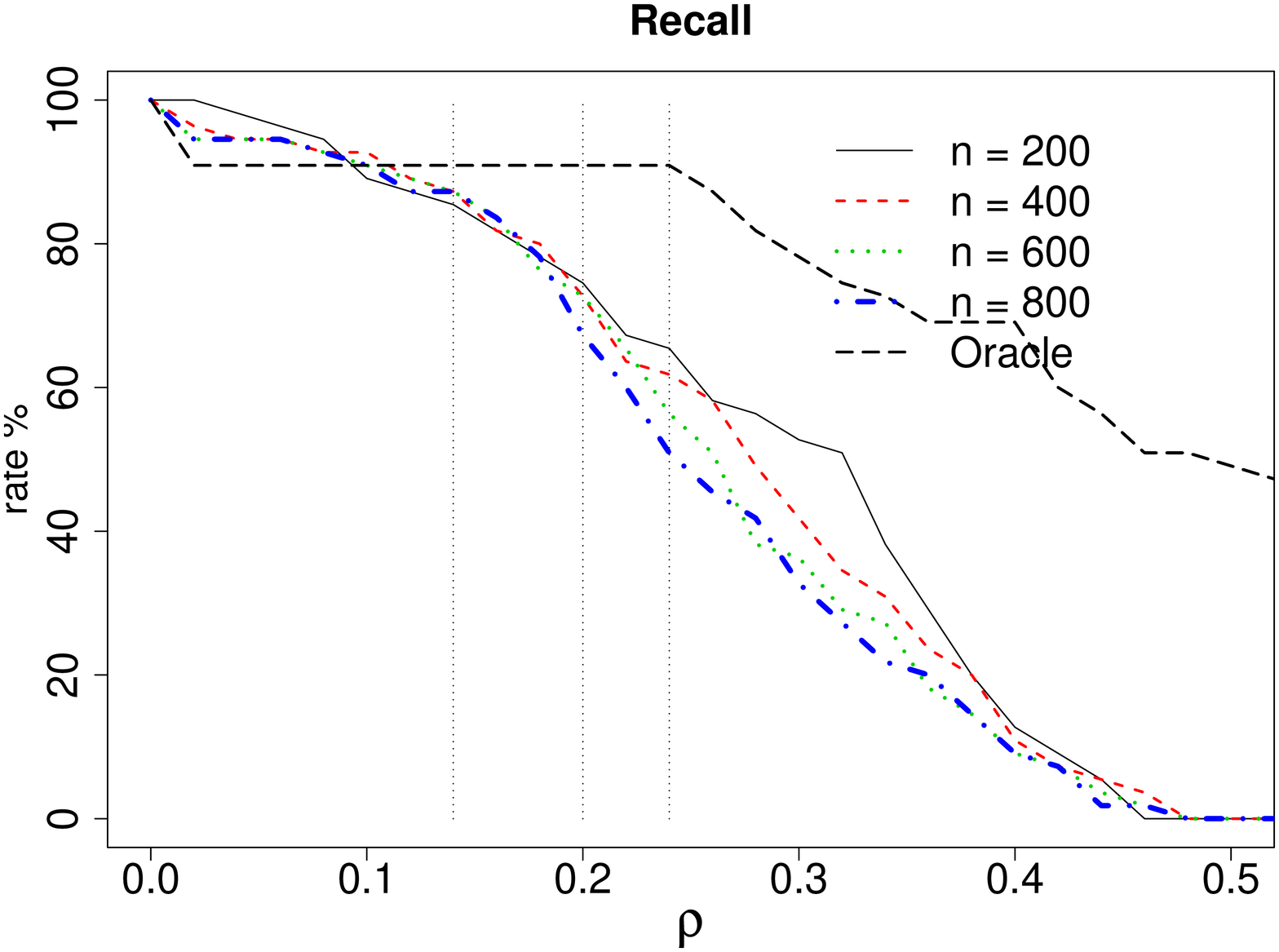}\\[-10pt]
\includegraphics[width=.30\textwidth,angle=-90]{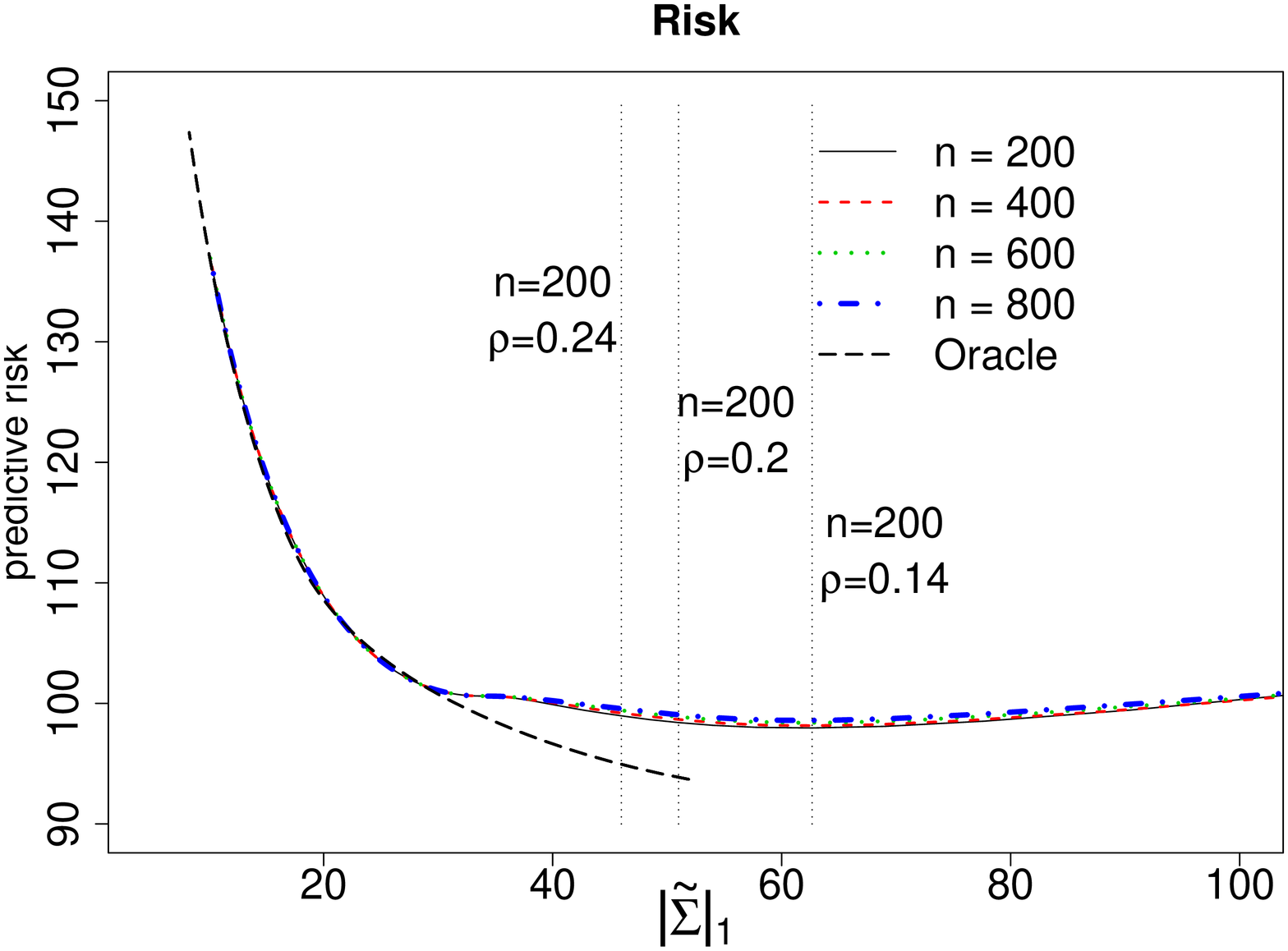}\\
\end{tabular}
\caption{Plots from top to bottom show that
as the penalization parameter $\rho$ increases, precision goes up, 
and then down as no edges are predicted in the end.  
Recall goes down as the estimated graphs are missing more
and more edges.  The oracle $\Sigma^*$ performs the best, given the
same value for $|\hat{\Sigma}_n(t_0)|_1 = |\Sigma^*|_1, \forall n$.
}
\end{center}
\label{fig:penalty}
\end{figure}
When we later delete an existing edge from
the graph, we reverse the above procedure with its weight.
Weights are assigned to the initial $50$ edges, and then we change
the graph structure periodically as follows: Every $200$ discrete time 
steps, five existing edges are deleted, and five new edges are added.
However, for each of the five new edges, a target weight is chosen,
and the weight on the edge is gradually changed over the ensuing $200$
time steps in order ensure smoothness.  Similarly, for each of the
five edges to be deleted, the weight gradually decays to zero over the
ensuing $200$ time steps. Thus, almost always, there are $55$ edges in
the graph and $10$ edges have weights that are varying smoothly.

\subsection{Regularization Paths}

We increase the sample size from $n = 200,$ to $400$, $600$, and $800$ 
and use a Gaussian kernel with bandwidth $h = \frac{5.848}{n^{1/3}}$. 
We use the following metrics to evaluate model
consistency risk for~$(\ref{eq::model-risk})$ and predictive 
risk~$(\ref{eq::future-risk})$ in Figure~$1$ as the $\ell_1$ regularization
parameter $\rho$ increases.
\begin{itemize}
\item
Let $\hat{F}_n$ denote edges in 
estimated $\hat{\Theta}_n(t_0)$ and $F$ denote edges in 
$\Theta(t_0)$. Let us define
\begin{eqnarray*}
\prec & = & 1 - \frac{\hat{F}_n \setminus F}{\hat{F}_n}
 = \frac{\hat{F}_n \cap F}{\hat{F}_n}, \\
\recall & = & 1 - \frac{F \setminus \hat{F}_n}{F}
 = \frac{\hat{F}_n \cap F}{F}.
\end{eqnarray*}
Figure~$1$ shows how they change with $\rho$.
\item
Predictive risks in~$(\ref{eq::future-risk})$ are plotted
for both the oracle estimator~$(\ref{eq::oracle-estimator})$ and 
empirical estimators~$(\ref{eq::emp-estimator})$ for each $n$.
They are indexed with the $\ell_1$ norm of various 
estimators vectorized; hence 
$|\cdot|_1$ for $\hat{\Sigma}_n(t_0)$ and $\Sigma^*(t_0)$ are the same along a 
vertical line. 
Note that $|\Sigma^*(t_0)|_1 \leq |\Sigma(t_0)|_1, \forall \rho \geq 0$;
for every estimator $\tilde\Sigma$ (the oracle or empirical), 
$|\tilde\Sigma|_1$ decreases as $\rho$ increases, as shown in Figure~$1$ 
for $|\hat\Sigma_{200}(t_0)|_1.$
\end{itemize}
Figure~2 shows a subsequence of estimated graphs 
as $\rho$ increases for sample size $n = 200$. The original graph at $t_0$ 
is shown in Figure~$1$.
\begin{figure}[ht]
\begin{center}
\begin{tabular}{ccc}
$G(p, \hat{F}_n)$ & $G(p, \hat{F}_n \setminus F)$ & $G(p, F \setminus \hat{F}_n)$ \\ 
\hspace{-10pt}
\includegraphics[width=.135\textwidth,angle=-0]{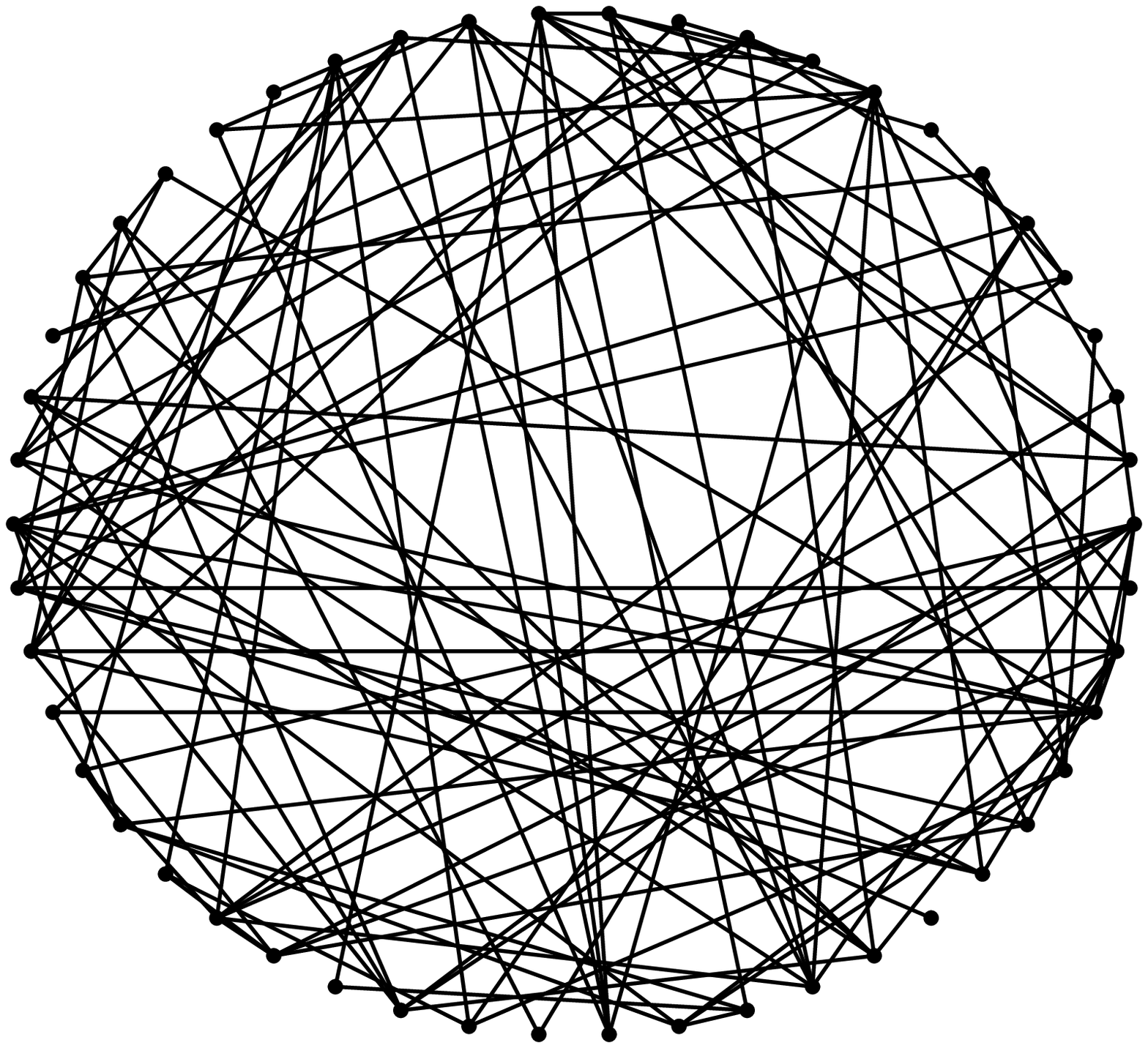}&
\hspace{-10pt}
\includegraphics[width=.135\textwidth,angle=-0]{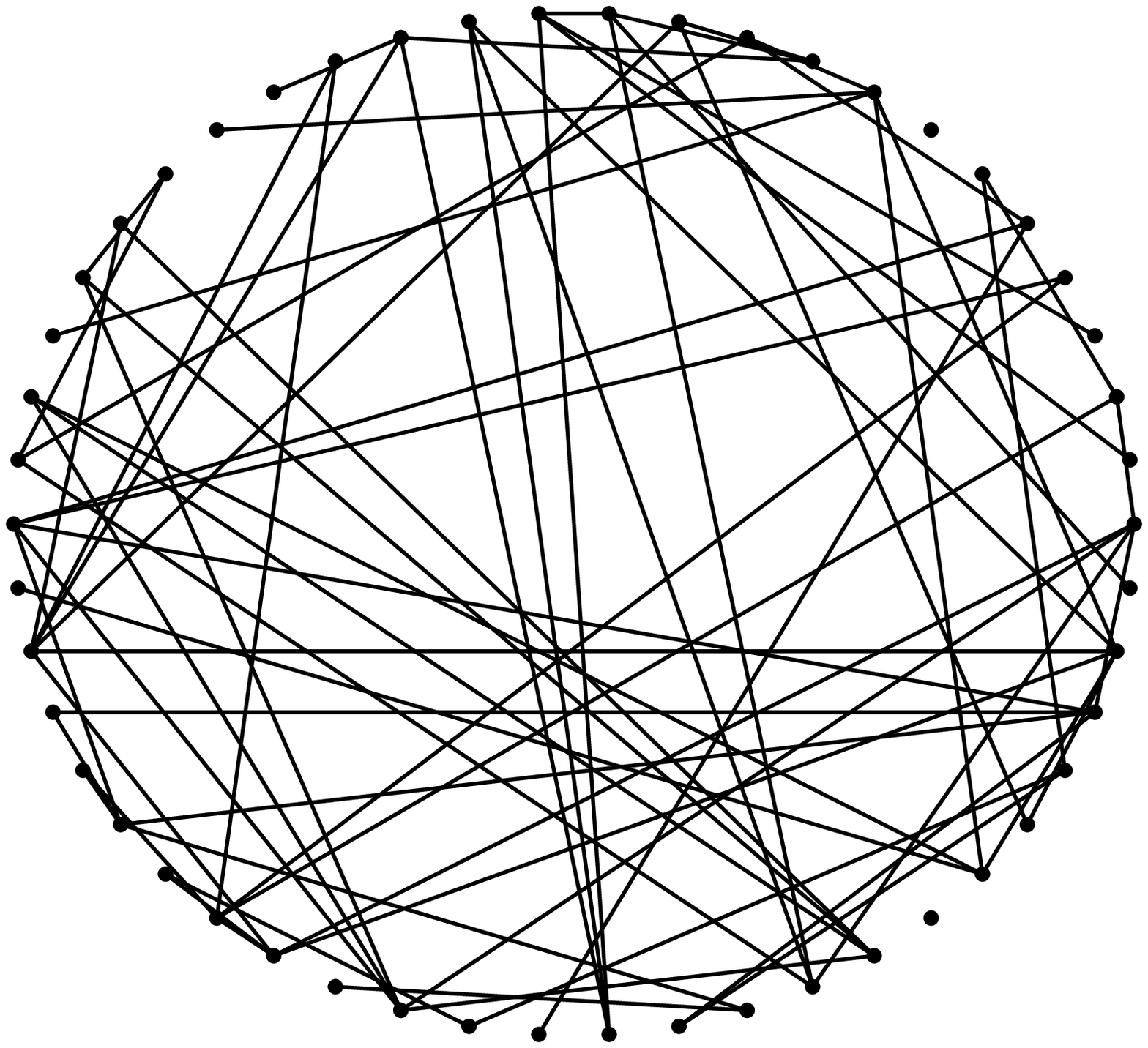}&
\hspace{-10pt}
\includegraphics[width=.135\textwidth,angle=-0]{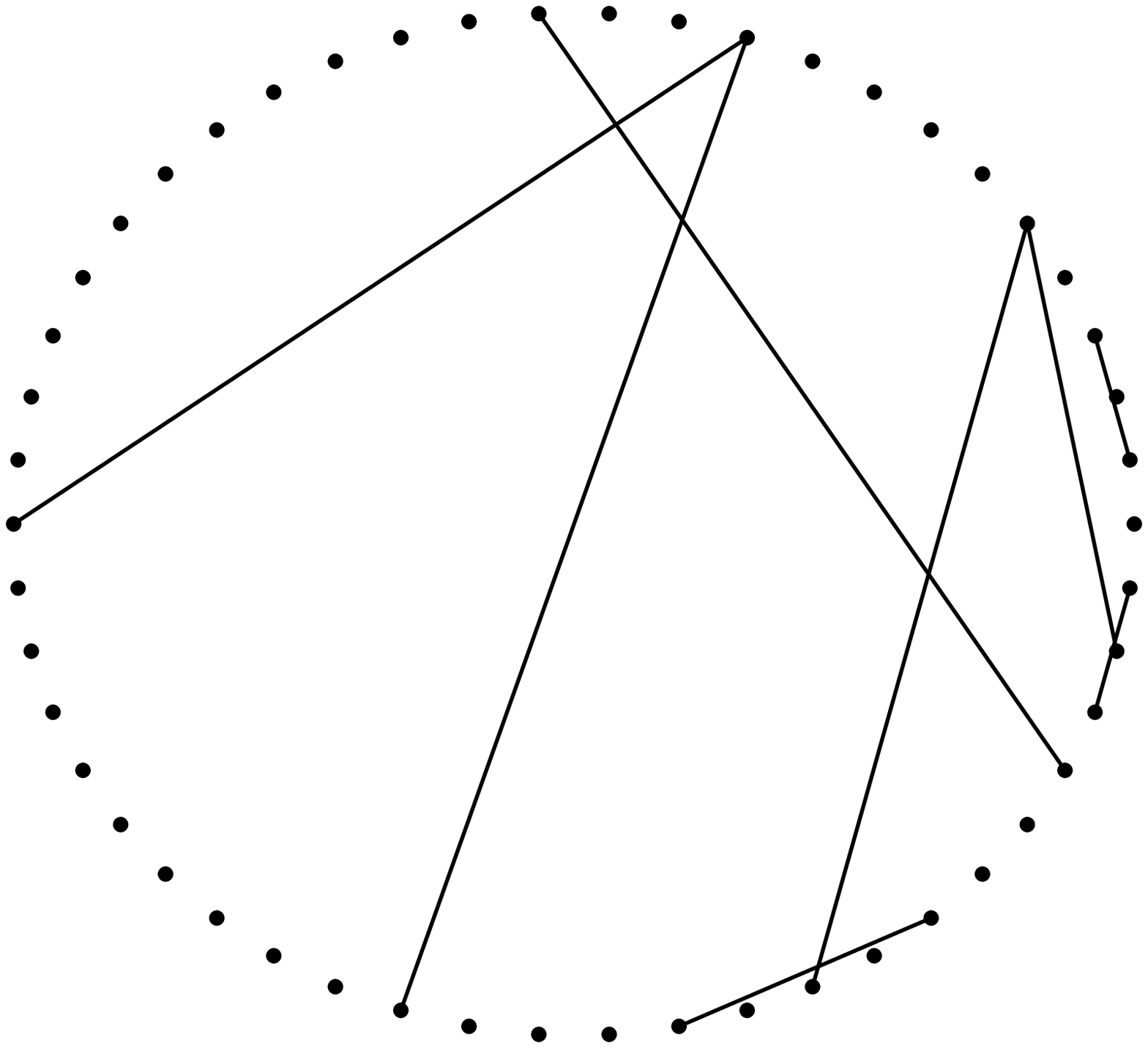}\\
\hspace{-10pt}
\includegraphics[width=.135\textwidth,angle=-0]{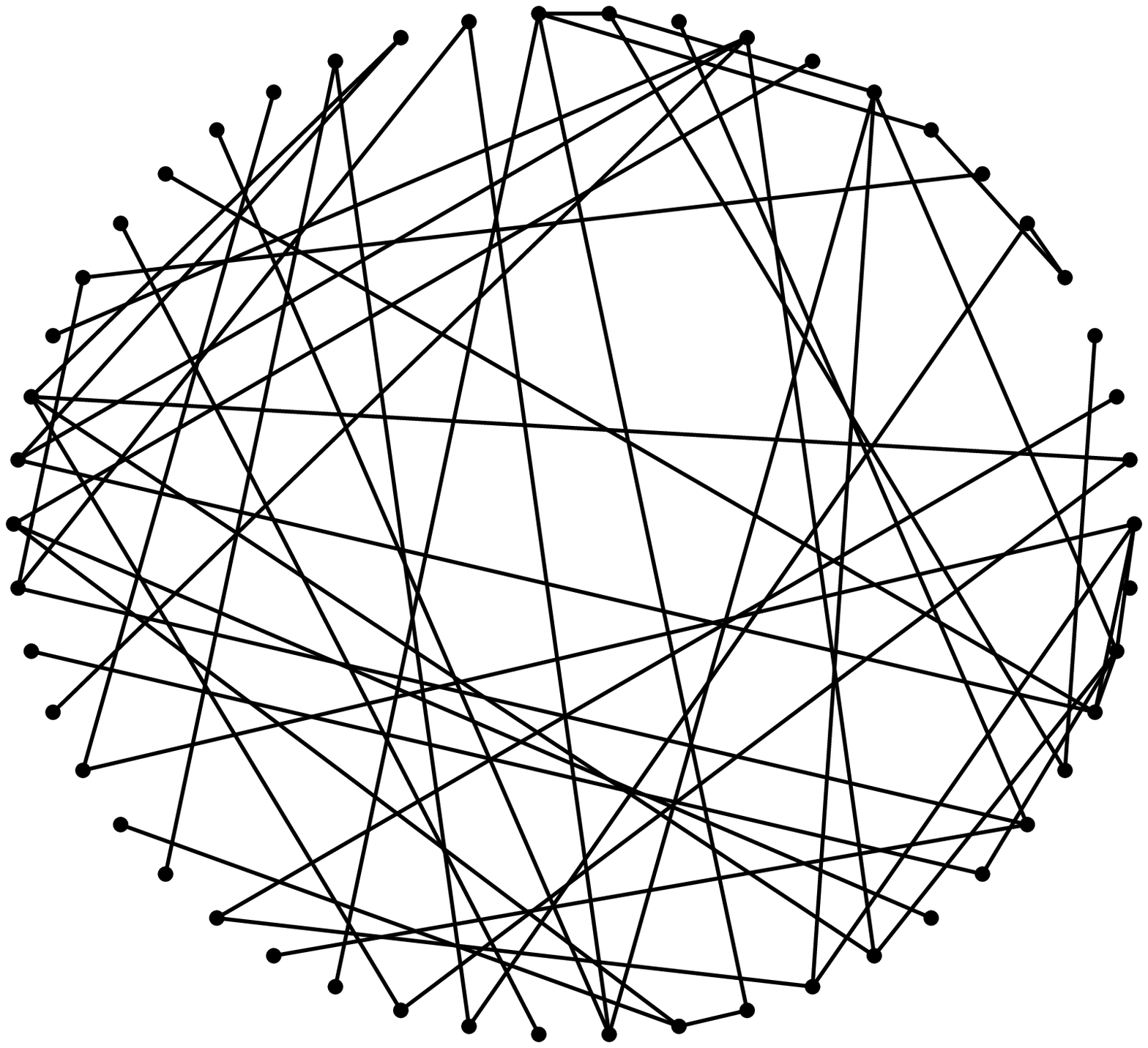}&
\hspace{-10pt}
\includegraphics[width=.135\textwidth,angle=-0]{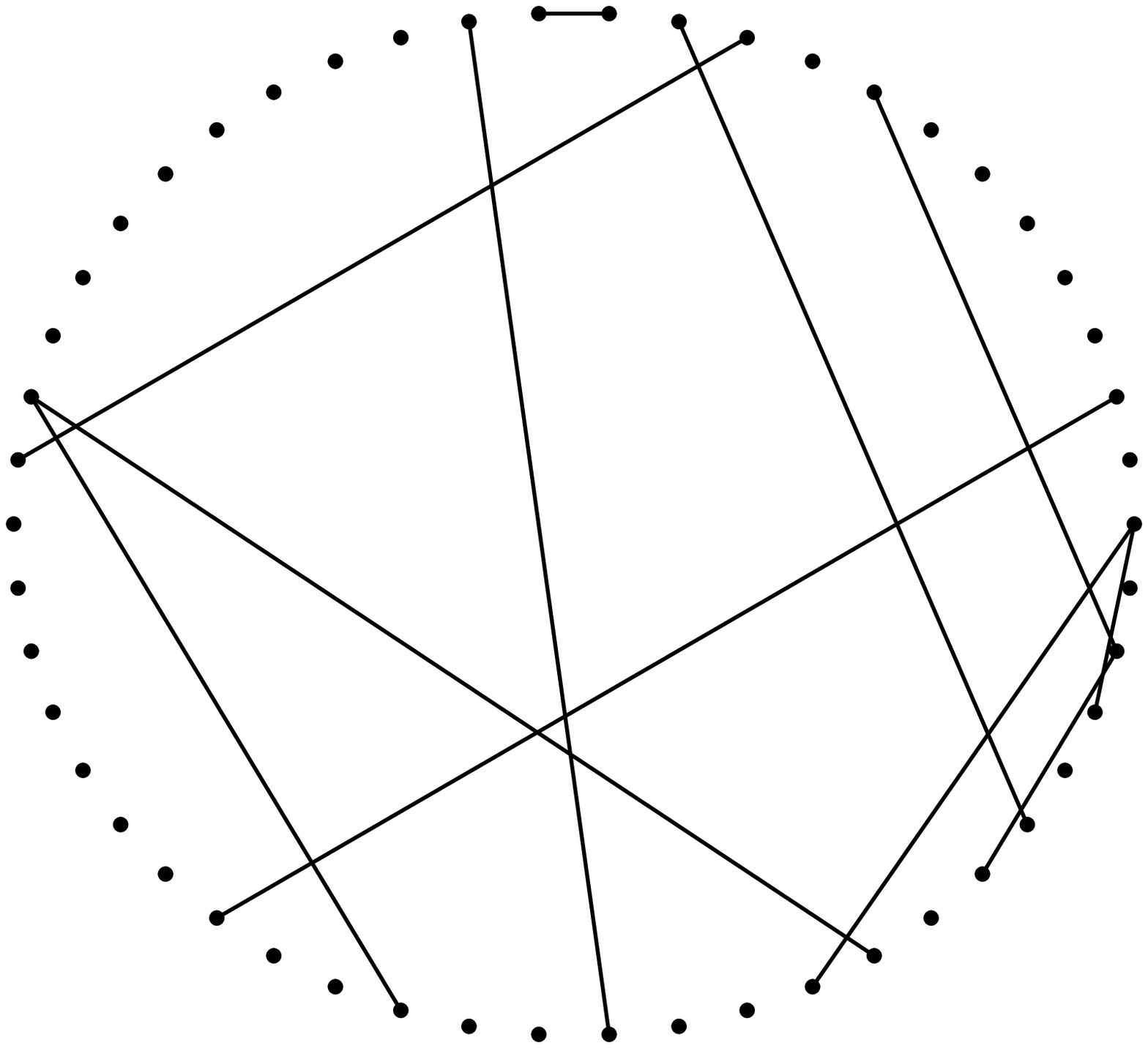}&
\hspace{-10pt}
\includegraphics[width=.135\textwidth,angle=-0]{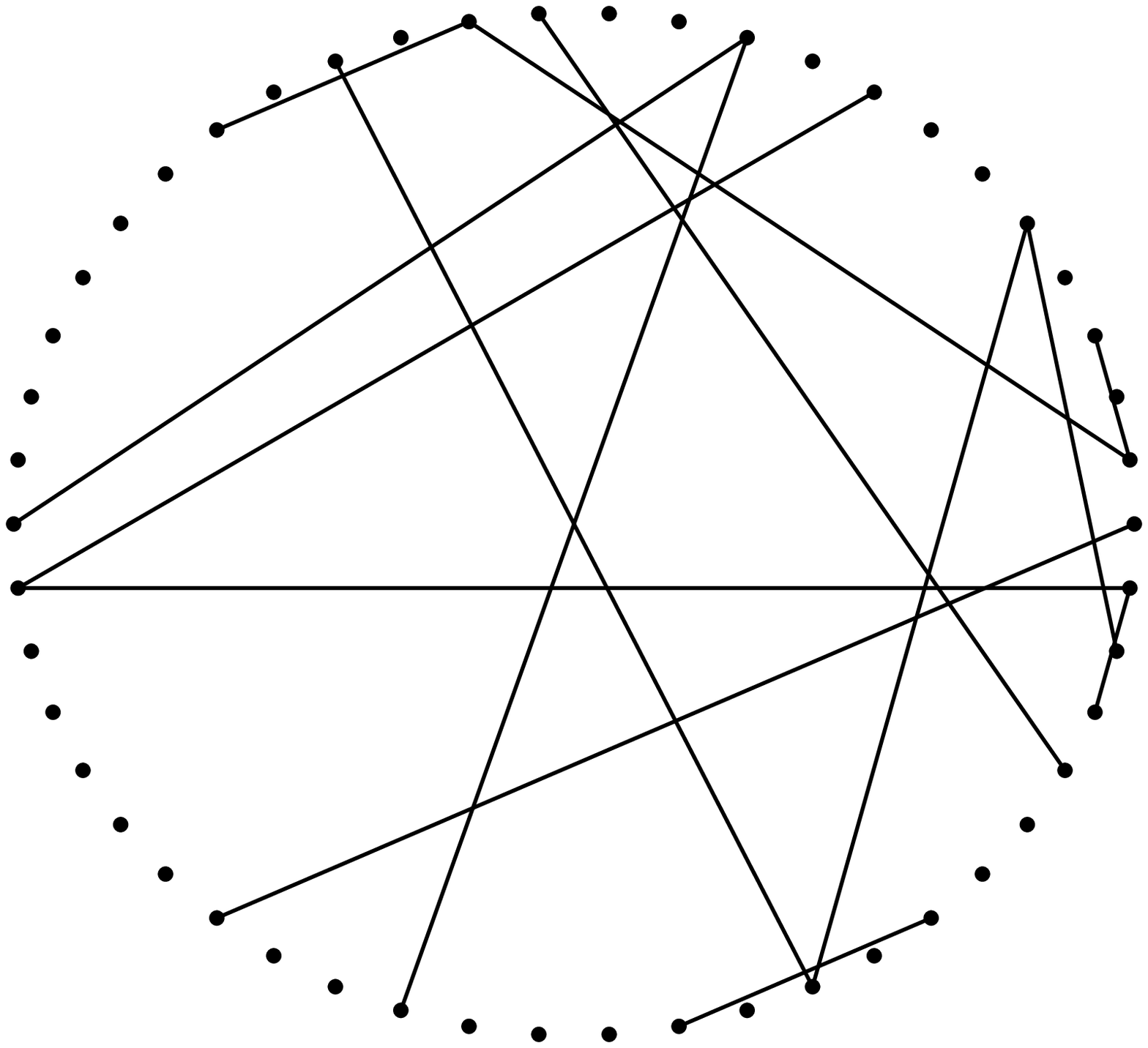} \\
\hspace{-10pt}
\includegraphics[width=.135\textwidth,angle=-0]{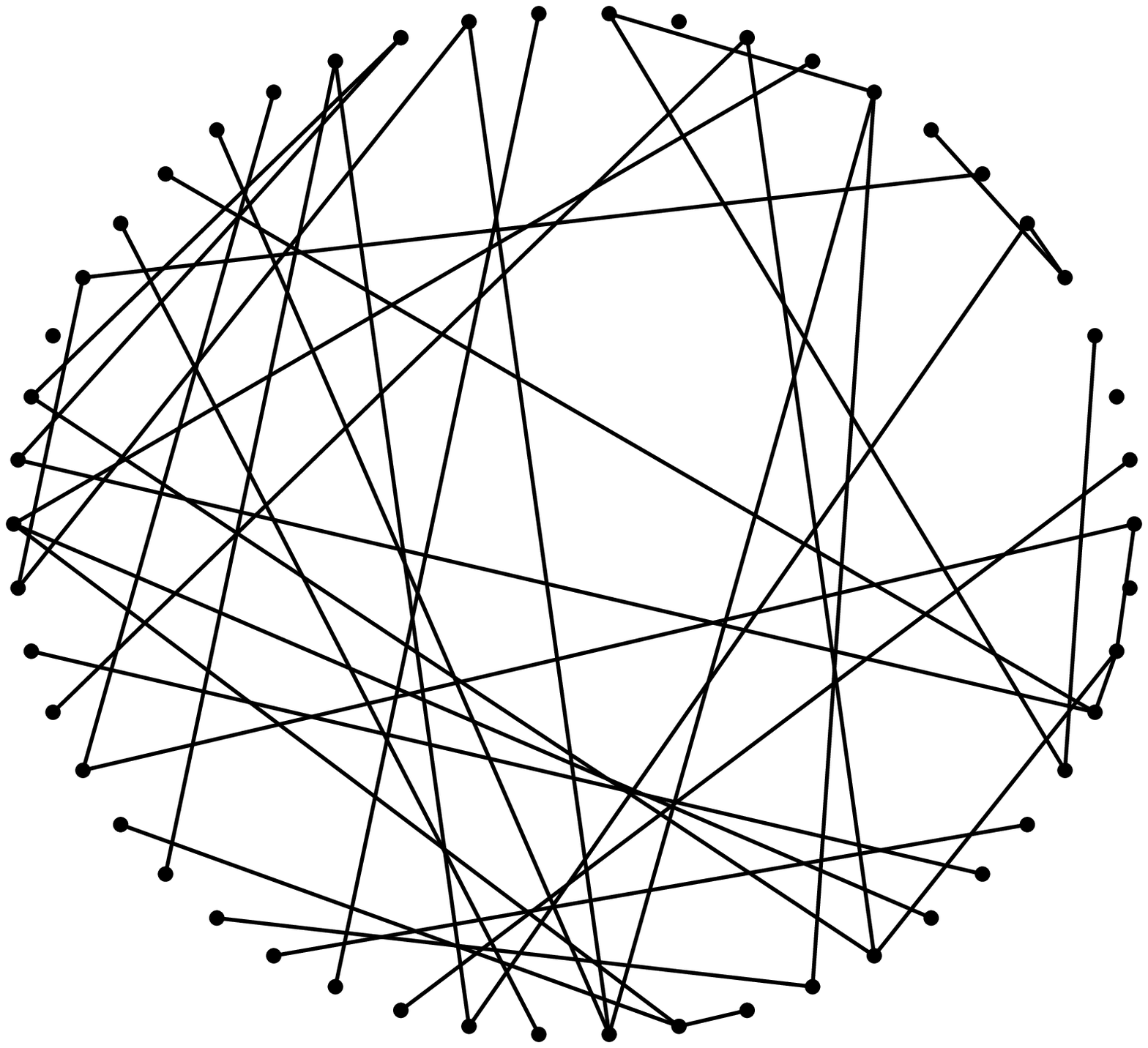}&
\hspace{-10pt}
\includegraphics[width=.135\textwidth,angle=-0]{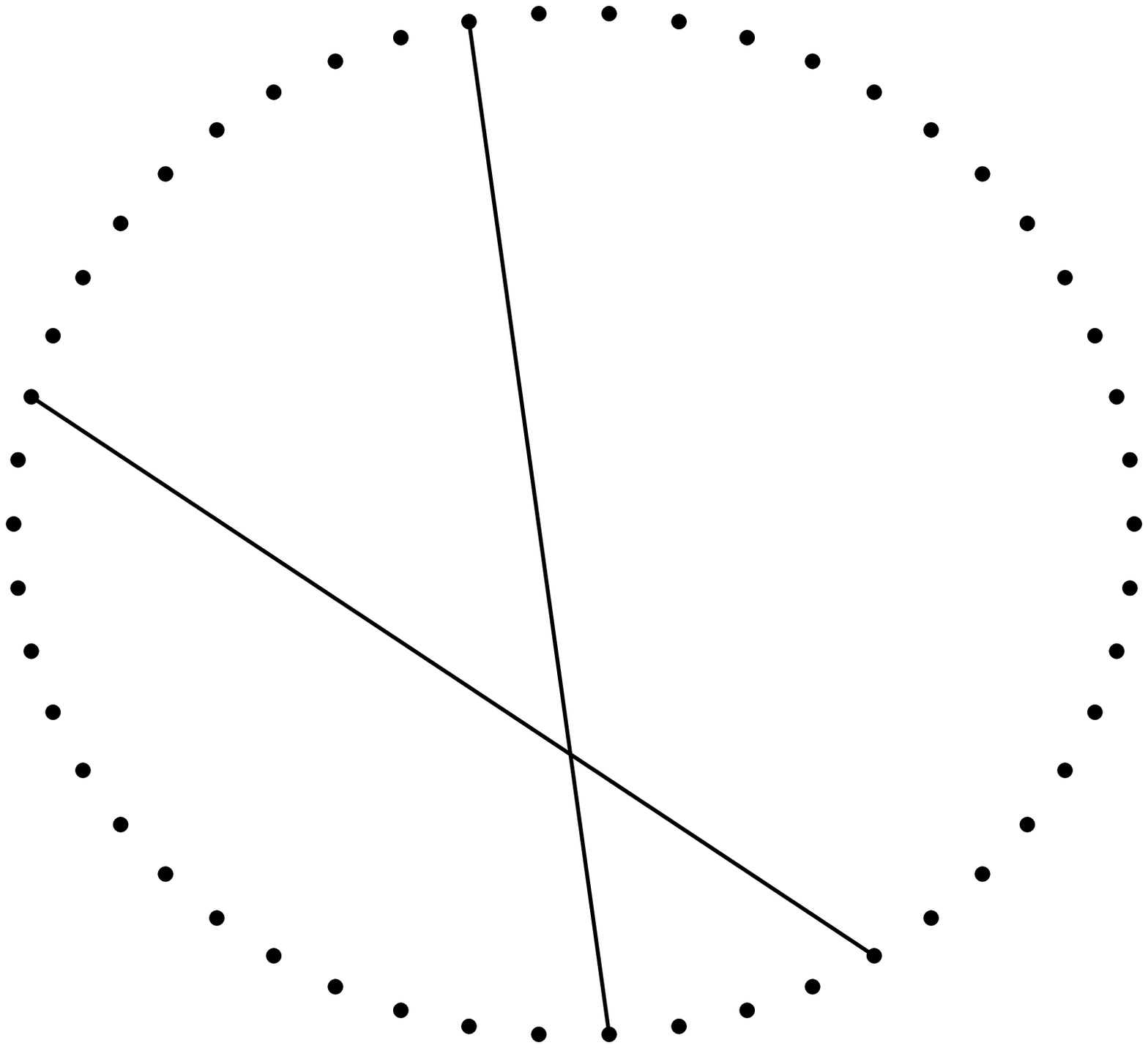}&
\hspace{-10pt}
\includegraphics[width=.135\textwidth,angle=-0]{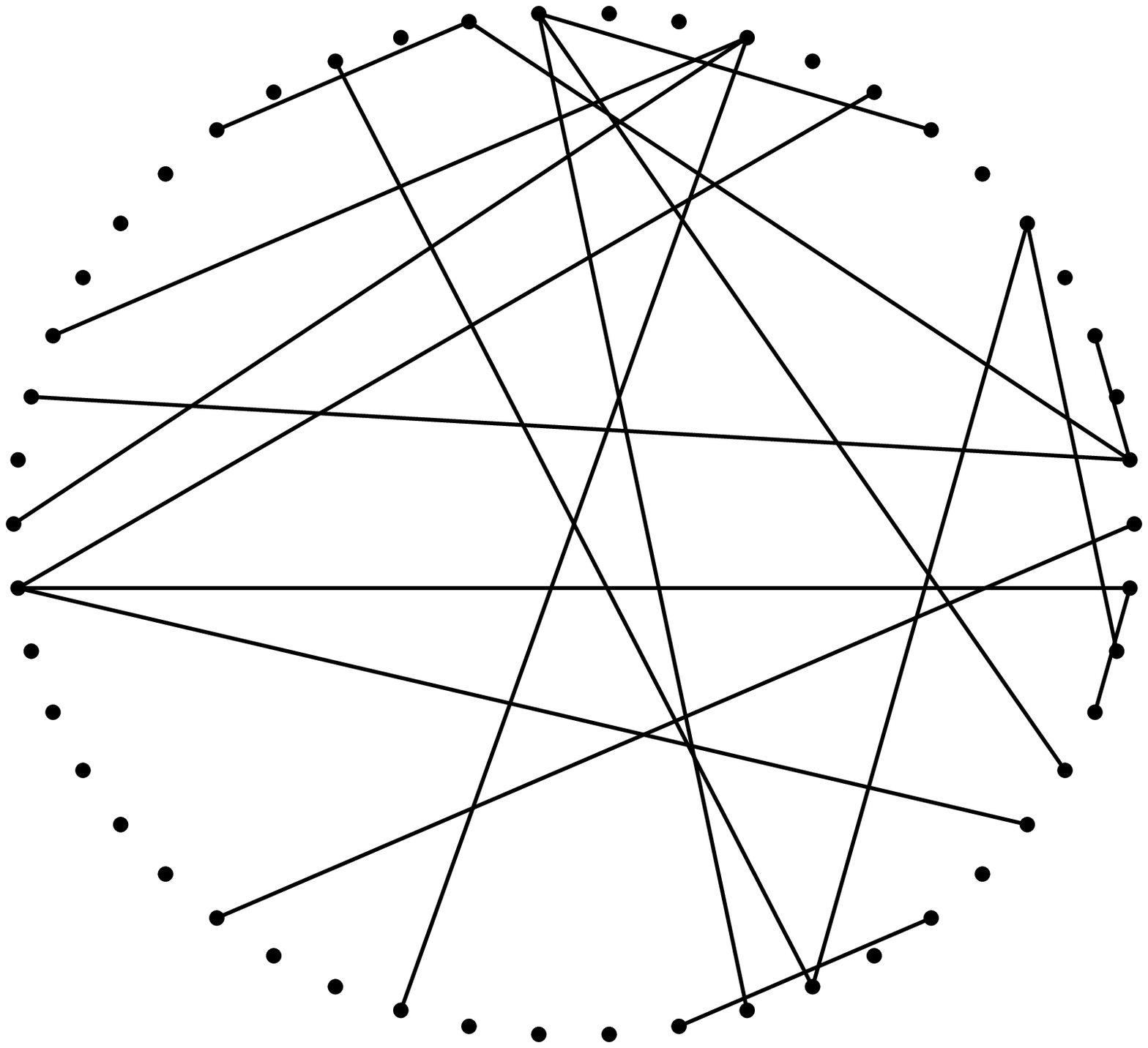}
\hspace{-10pt}
\end{tabular}
\caption{
$n = 200$ and $h = 1$ with $\rho = 0.14, 0.2, 0.24$ indexing each row.
The three columns show sets of edges in $\hat{F}_n$,
extra edges, and missing edges with respect to the true graph $G(p, F)$.
This array of plots show that $\ell_1$ regularization is effective in selecting
the subset of edges in the true model $\Theta(t_0)$, even when the samples 
before $t_0$ were from graphs that evolved over time. }
\end{center}
\label{fig:path}
\end{figure}

\subsection{Chasing the Changes}
Finally, we show how quickly the smoothed estimator using GLASSO~\cite{FHT07} 
can include the edges that are being added in the beginning of interval $[0, 1]$,
and get rid of edges being replaced, whose weights start to decrease at $x = 0$ and
become $0$ at $x = 0.5$ in Figure~$3$.
\begin{figure*}[htbp]
\begin{center}
\begin{tabular}{cccccc}
Edges &&&&&\\
\multirow{2}{*}{\includegraphics[width=.2\textwidth,angle=-0]{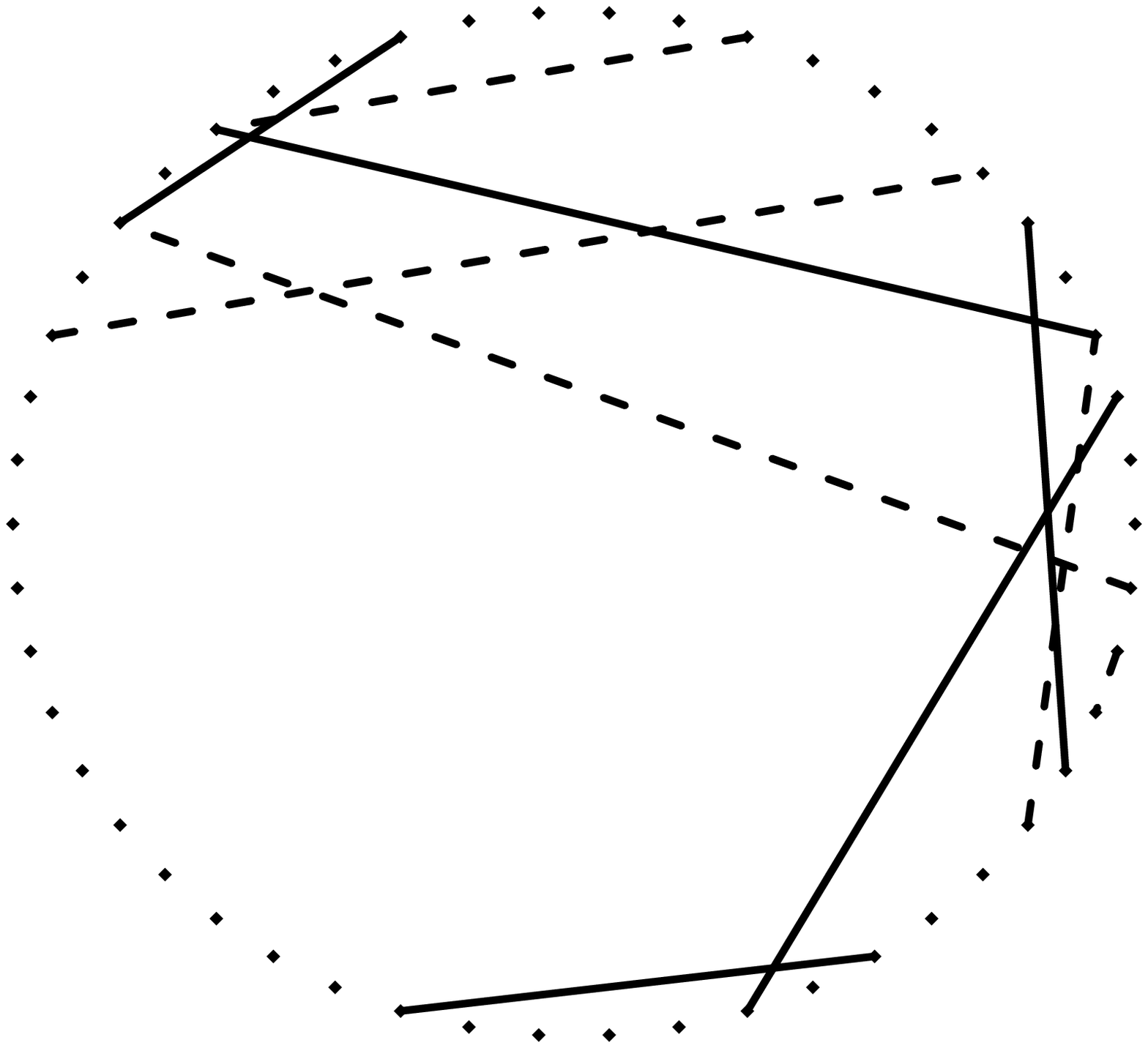}}
&
0.35 & 0.4875 & 0.52 & 0.5275 & 0.6125 \\
&
\includegraphics[width=.12\textwidth,angle=-0]{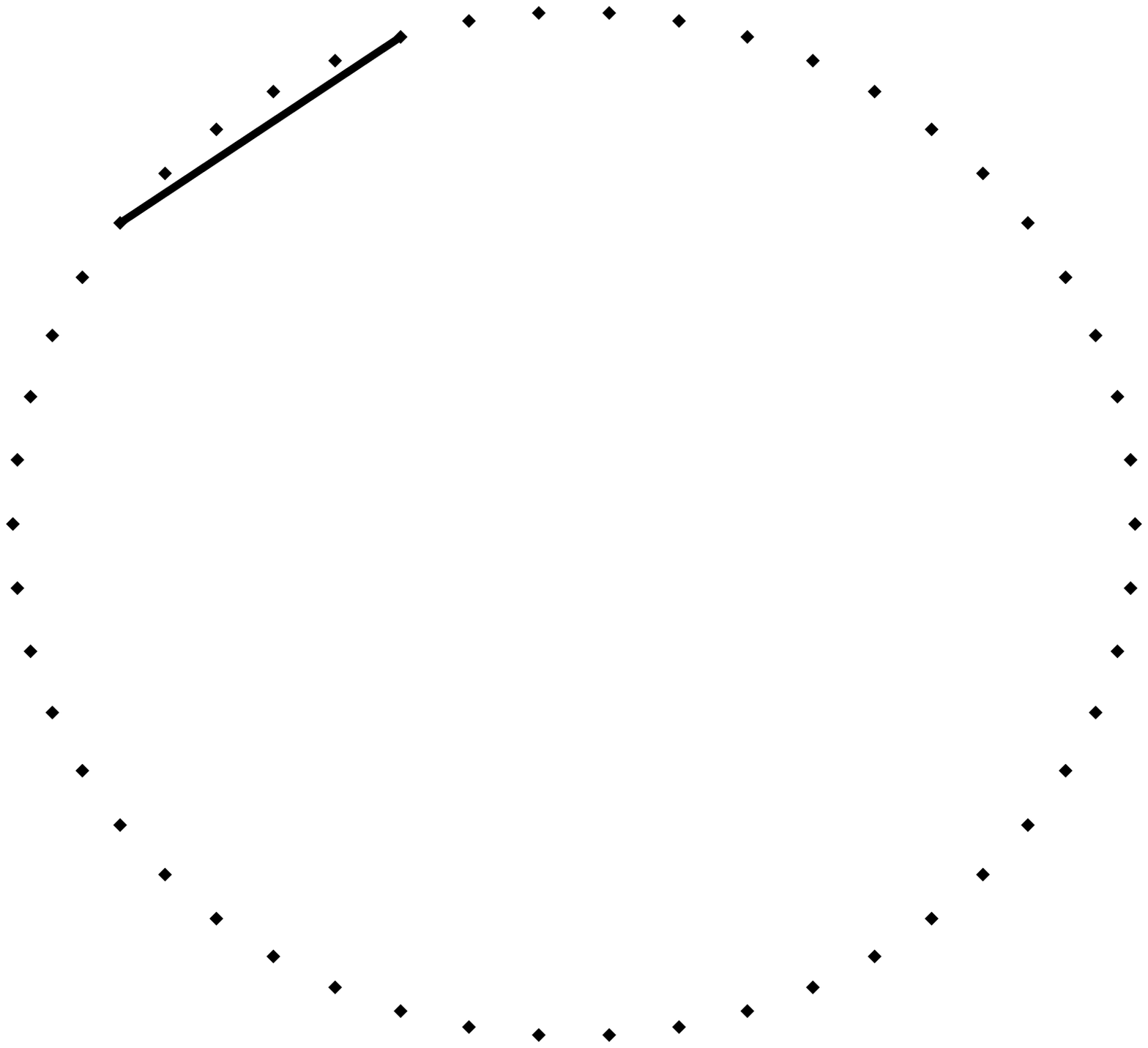}&
\includegraphics[width=.12\textwidth,angle=-0]{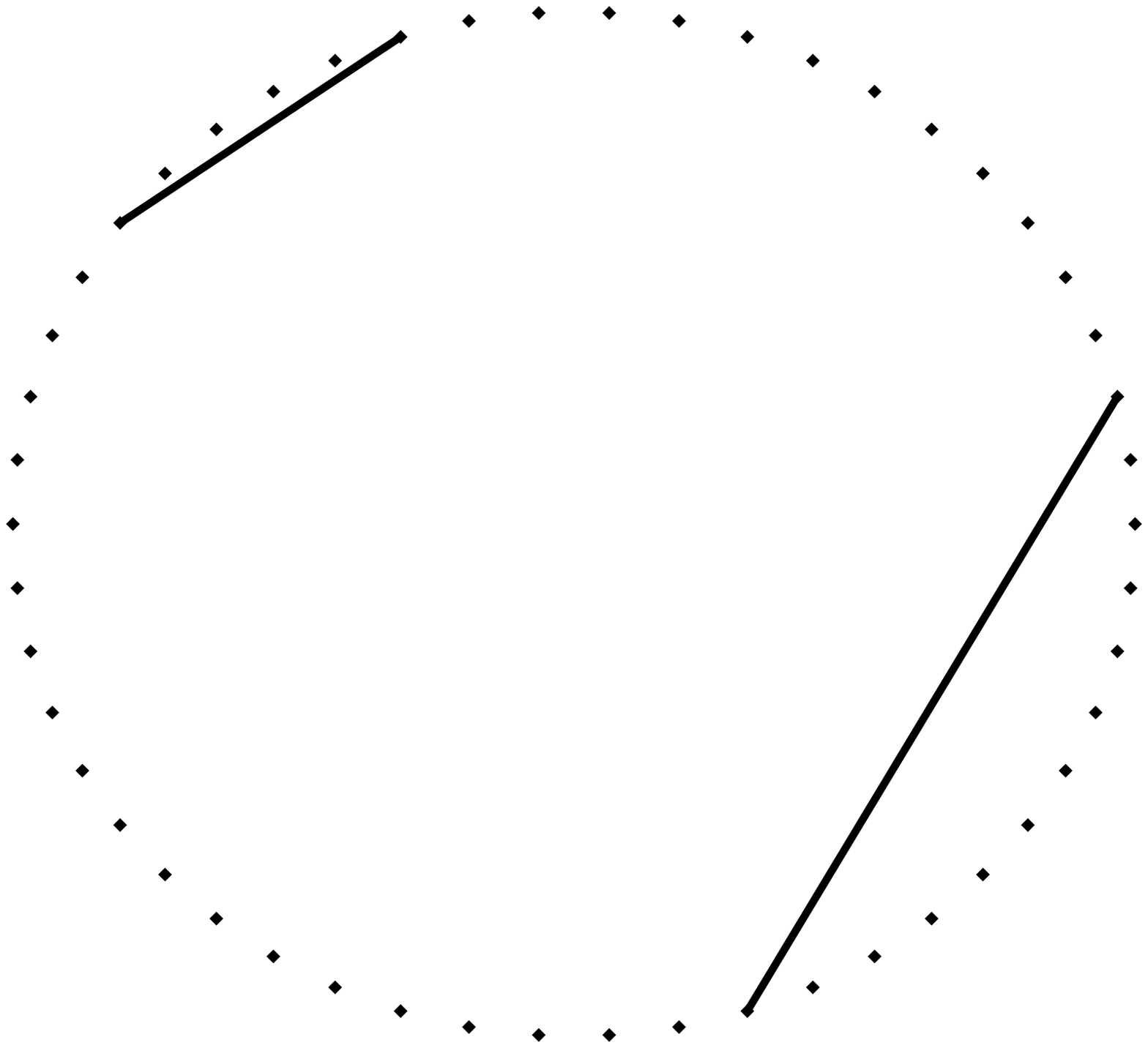}&
\includegraphics[width=.12\textwidth,angle=-0]{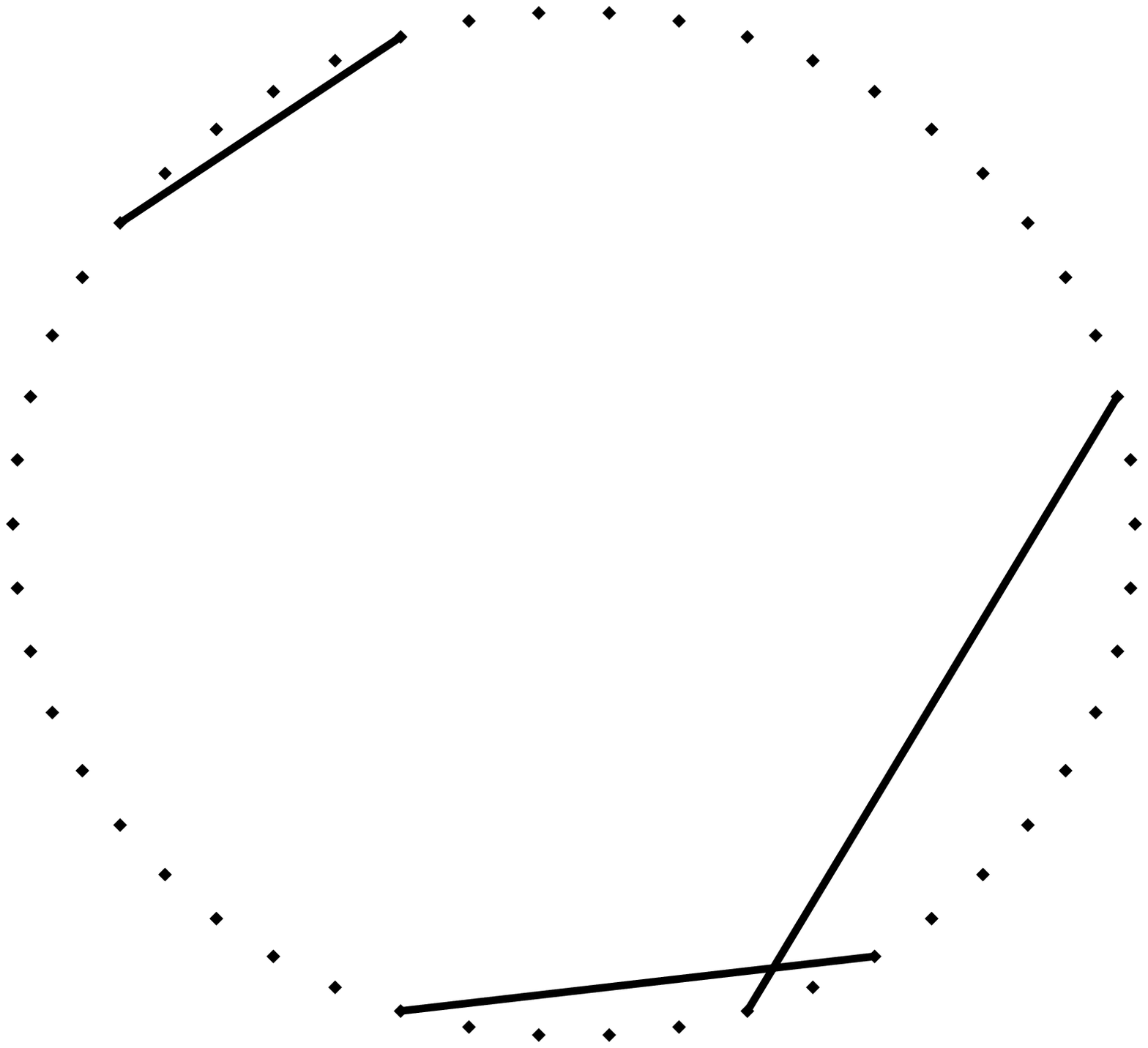}&
\includegraphics[width=.12\textwidth,angle=-0]{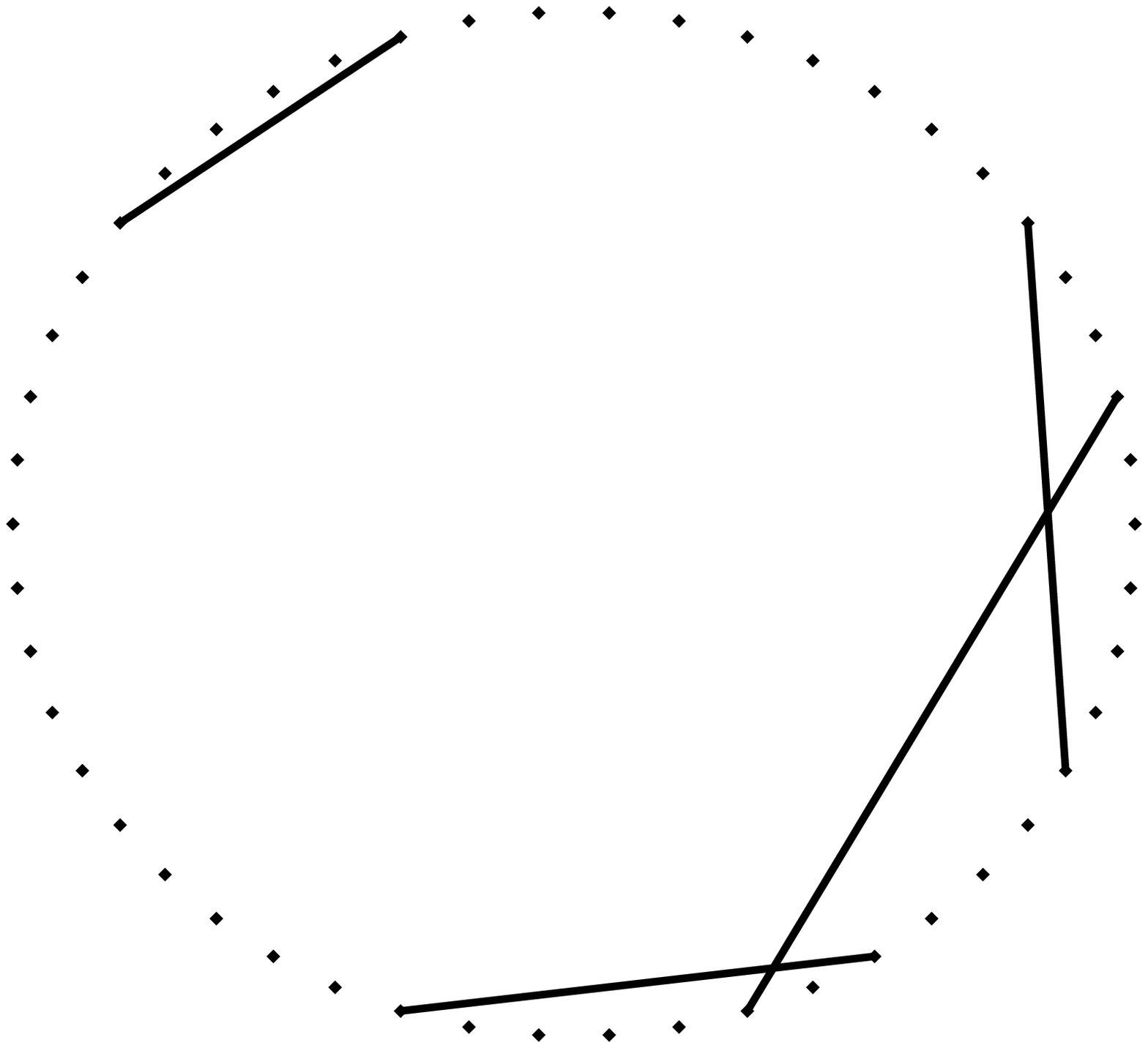}&
\includegraphics[width=.12\textwidth,angle=-0]{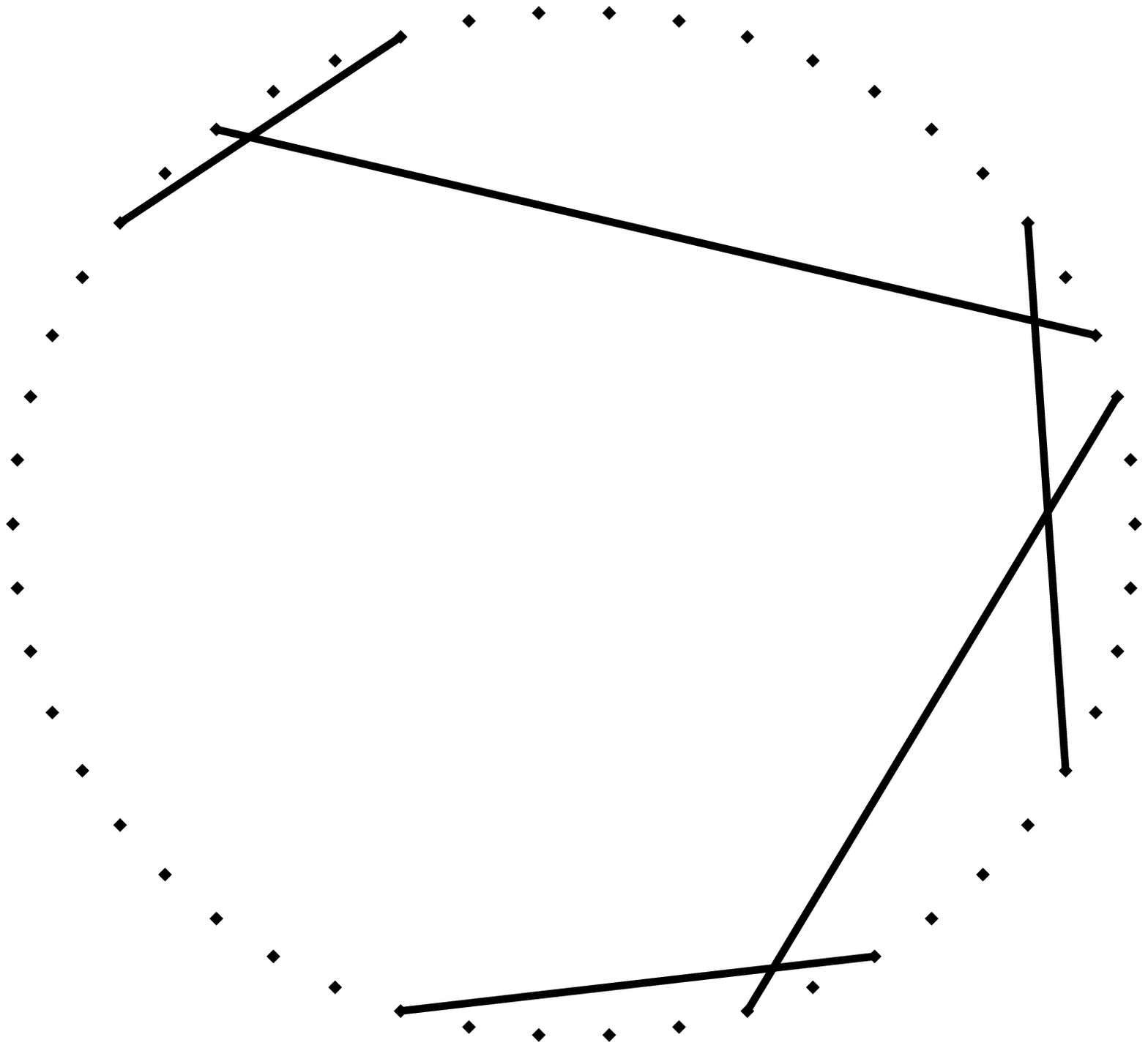}\\
&
0.02& 0.0825 & 0.1275 & 0.21 & 0.595 \\
&
\includegraphics[width=.12\textwidth,angle=-0]{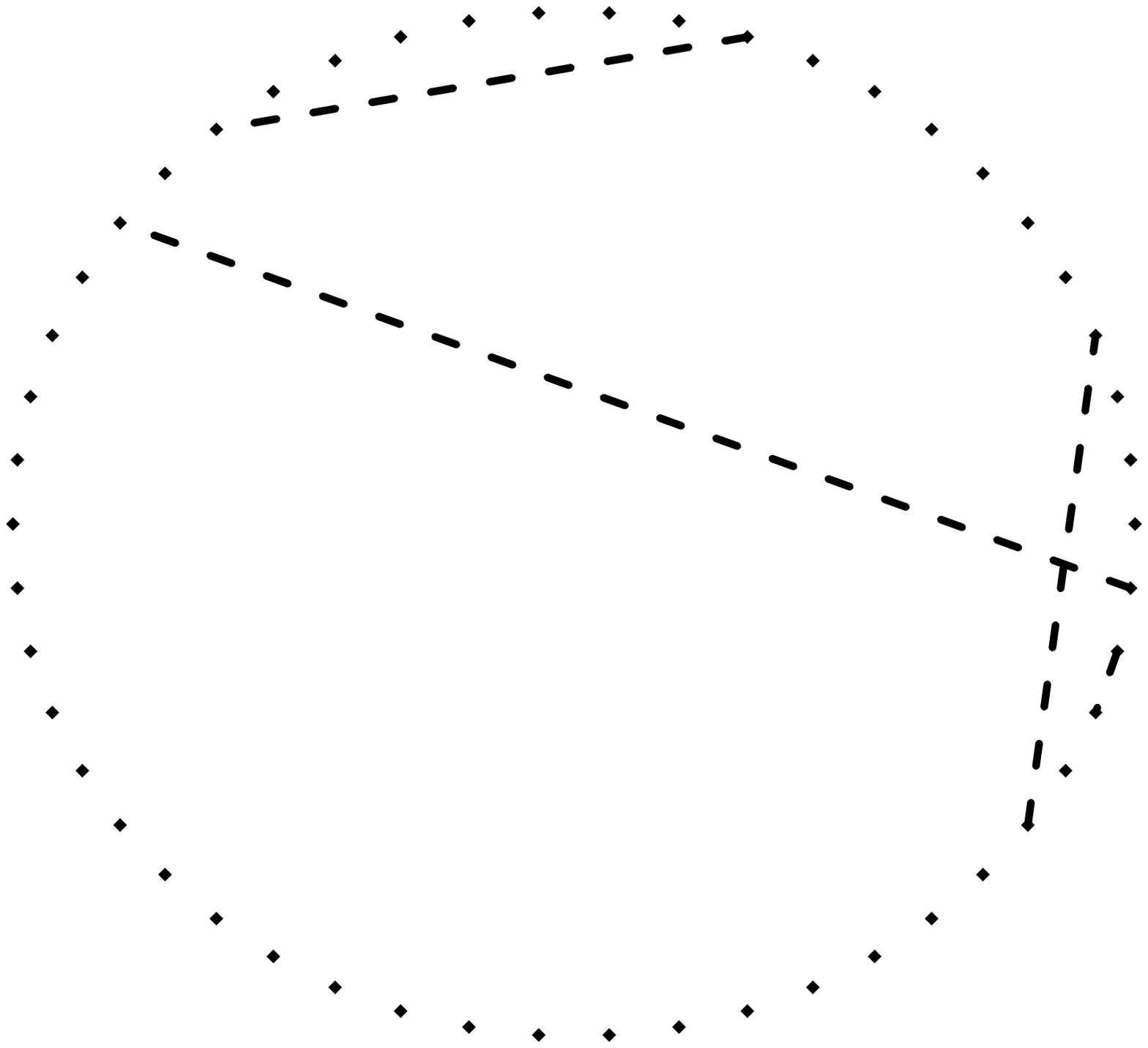}&
\includegraphics[width=.12\textwidth,angle=-0]{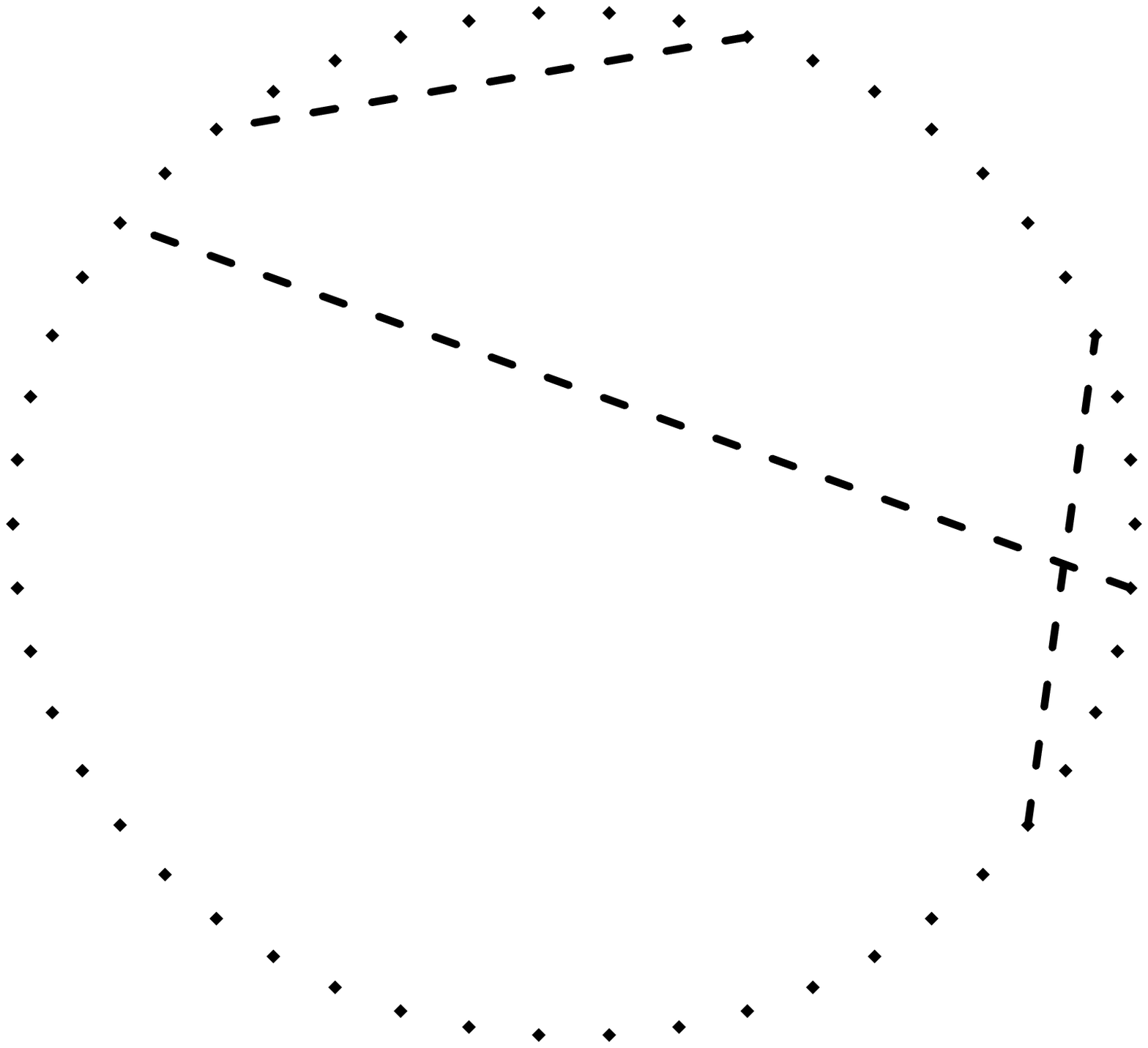}&
\includegraphics[width=.12\textwidth,angle=-0]{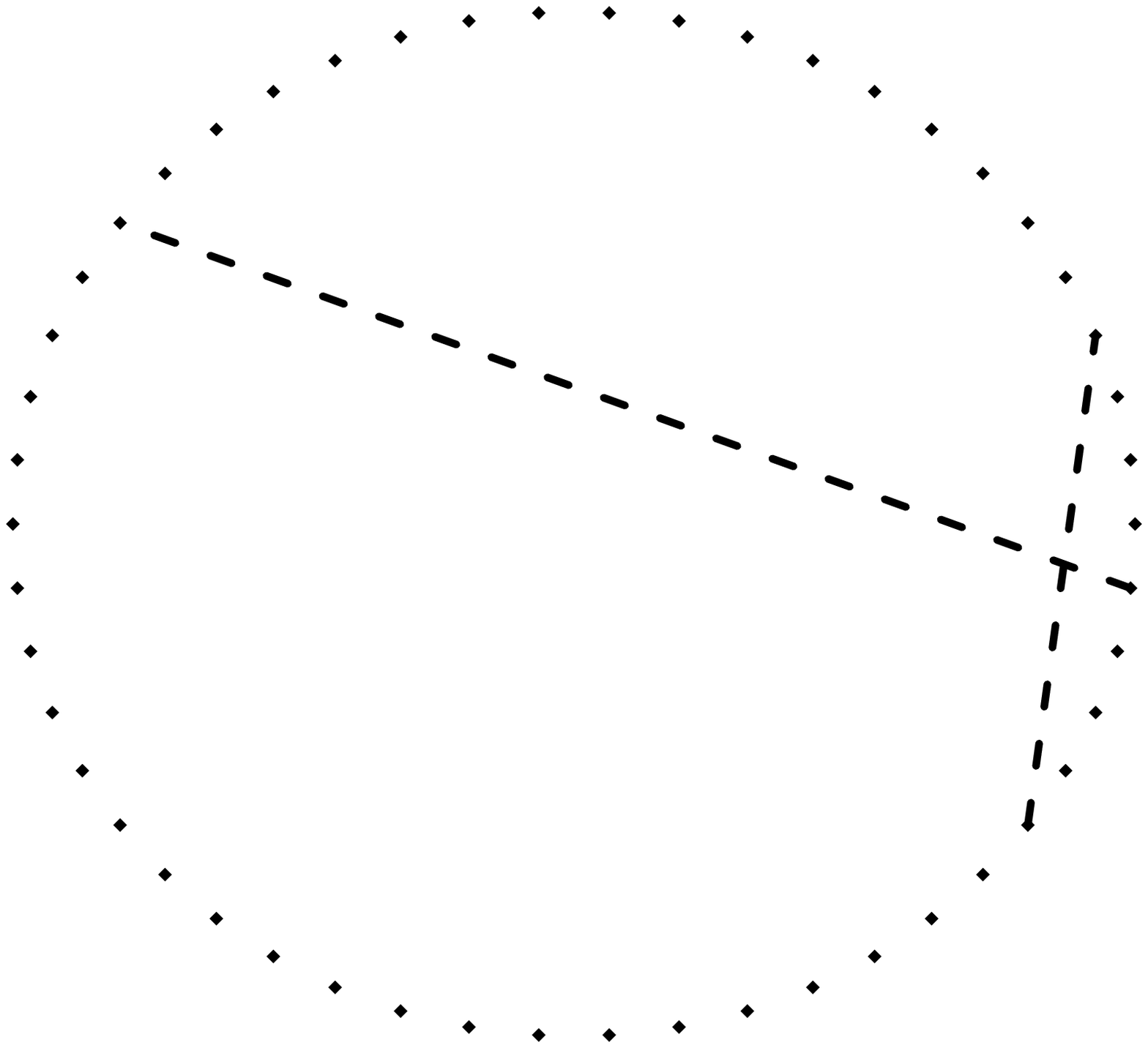}&
\includegraphics[width=.12\textwidth,angle=-0]{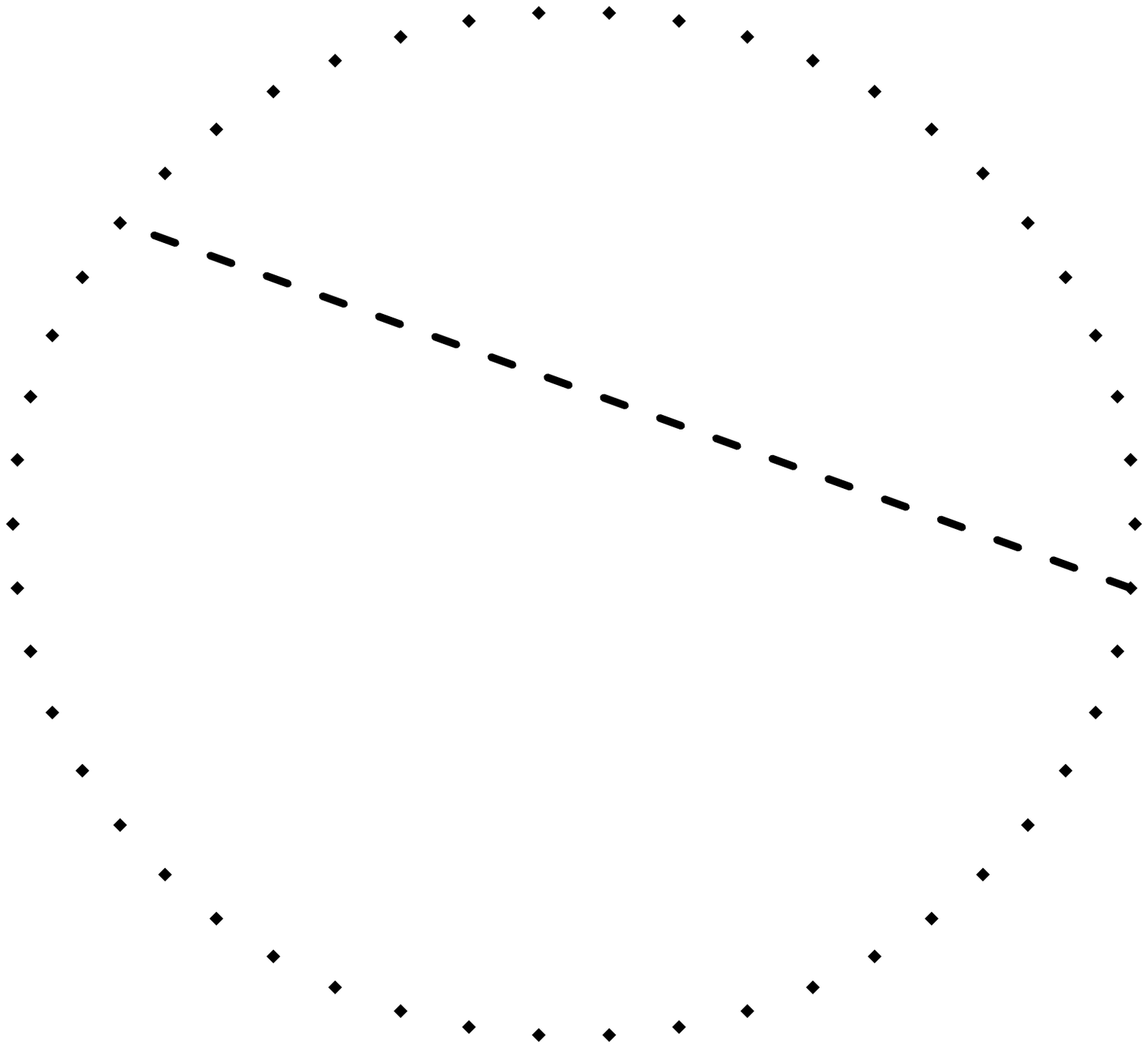}&
\includegraphics[width=.12\textwidth,angle=-0]{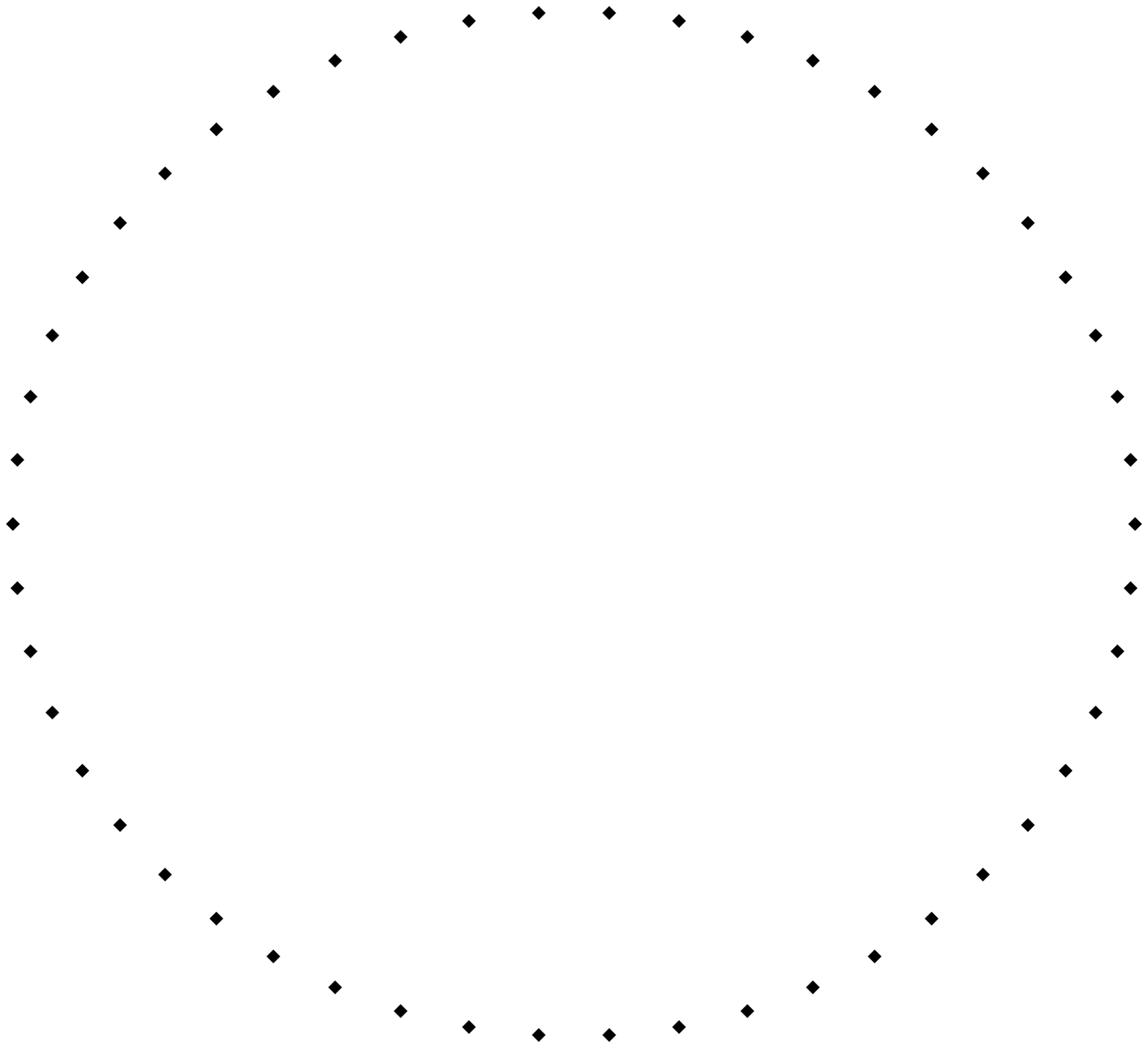}
\end{tabular}
\caption{There are $400$ discrete steps in $[0,1]$ such that 
the edge set $F(t)$ remains unchanged before or after $t = 0.5$.
This sequence of plots shows the times
at which each of the new edges added at $t=0$ appears in the
estimated graph (top row), and the times at which each of the old
edges being replaced is removed from the estimated graph (bottom
row), where the weight decreases from a positive value in $[0.1,0.3]$ 
to zero during the time interval $[0,0.5]$. 
Solid and dashed lines denote new and old edges respectively.}
\end{center}
\label{fig:new-old}
\vskip-10pt
\end{figure*}

\section{Conclusions and Extensions}

We have shown that if the covariance changes smoothly
over time, then minimizing an $\ell_1$-penalized
kernel risk function leads to good estimates of the
covariance matrix.
This, in turn, allows estimation of time varying graphical structure.
The method is easy to apply and is feasible in high dimensions.

We are currently addressing several extensions to this work.
First, with stronger conditions 
we expect that we can establish {\em sparsistency}, that is,
we recover the edges  with probability approaching one.
Second, we can relax the
smoothness assumption using nonparametric changepoint methods
~\cite{GH02} which allow for jumps.
Third, we used a very simple time series model; extensions to 
more general time series models
are certainly feasible.

\bibliography{local} \bibliographystyle{alpha}
\appendix
\appendix
\section{Large Deviation Inequalities for Boxcar Kernel Function}
\label{sec:append-boxcar}
In this section, we prove the following lemma, which implies the
i.i.d case as in the corollary.
\begin{lemma}
\label{lemma:boxcar-deviation}
Using a boxcar kernel that weighs uniformly over $n$ samples 
$Z_k \sim N(0, \Sigma(k)), k = 1, \ldots, n$, that are independently but
not identically distributed, we have for $\epsilon$ small enough, for some $c_2>0$, 
$$
\prob{|\hat{S}_n(t, i, j) - \expct{\hat{S}_n(t, i, j)}| > \epsilon}
\leq \exp\left\{ -c_2 n \epsilon^2 \right\}.
$$
\end{lemma}
\begin{corollary}
\label{cor:iid-deviation}
For the i.i.d. case, for some $c_3 >0$,
$$
\prob{|\hat{S}_n(i, j) - \expct{\hat{S}_n(i, j)}| > \epsilon}
\leq \exp\left\{ - c_3 n \epsilon^2 \right\}.
$$
\end{corollary}
Lemma~\ref{lemma:boxcar-deviation} is implied by
Lemma~\ref{lemma:boxcar-deviation-diag} for diagonal entries, 
and Lemma~\ref{lemma:boxcar-deviation-non-diag} for non-diagonal entries.

\subsection{Inequalities for Squared Sum of Independent Normals with Changing Variances}
Throughout this section, we use $\sigma^2_{i}$ as a shorthand for $\sigma_{ii}$ as 
before. Hence $\sigma_i^2(x_k) = \var(Z_{k,i}) = \sigma_{ii}(x_k), 
\forall k =1 ,\ldots, n$.
Ignoring the bias term as in \eqref{eq::bd-decompose}, 
we wish to show that each of the diagonal entries of 
$\hat\Sigma_{ii}$ is close to $\sigma^2_{i}(x_0), \forall i=1, \ldots, p$.
For a boxcar kernel that weighs uniformly over $n$ samples, we mean strictly
$\ell_k(x_0) = \frac{1}{n}, \forall k =1, \ldots, n,$ and $h =1$ 
for \eqref{eq::kernel-weight} in this context.
We omit the mention of $i$ or $t$ in all symbols from here on.
The following lemma might be of its independent interest; hence we include it here.
We omit the proof due to its similarity to that of Lemma~\ref{lemma:deviation}.
\begin{lemma}
\label{lemma:boxcar-deviation-diag}
We let $z_1, \ldots, z_n$ represent a sequence of independent 
Gaussian random variables such that $z_k \sim N(0, \sigma^2(x_k))$. Let
$\sigma^2 = \frac{1}{n}\sum_{k=1}^n \sigma^2(x_k)$. 
Using a boxcar kernel that weighs uniformly over $n$ samples, 
$\forall \epsilon < c\sigma^2$,
for some $c \geq 2$, we have 
$$
\prob{\size{\inv{n} \sum_{k=1}^n z_{k}^2 - \sigma^2} > \epsilon}
\leq \exp\left\{ \frac{-(3c - 5)n\epsilon^2}{3 c^2 \sigma^2\sigma_{\max}^2} \right\},
$$
where $\sigma_{\max}^2 = \max_{k=1, \ldots, n}\{\sigma^2(x_k)\}$.
\end{lemma}

\subsection{Inequalities for Independent Sum of Products of Correlated Normals}
The proof of Lemma~\ref{lemma:boxcar-deviation-non-diag} follows that of 
Lemma~\ref{lemma:deviation}.
\begin{lemma}
\label{lemma:boxcar-deviation-non-diag}
Let $\Psi_2 = \inv{n} \sum_{k = 1}^n 
\frac{(\sigma_i^2(x_k)\sigma_j^2(x_k) + \sigma^2_{ij}(x_k))}{2}$ and
$c_4 = \frac{3}{20 \Psi_2}$.
Using a boxcar kernel that weighs uniformly over $n$ samples, for
$\epsilon \leq \frac{\Psi_2}{\max_{k}(\sigma_{i}(x_k) \sigma_{j}(x_k))},$
$$
\prob{|\hat{S}_n(t, i, j) - \expct{\hat{S}_n(t, i, j)}| > \epsilon
}
\leq \exp\left\{-c_4 n\epsilon^2 \right\}.$$
\end{lemma}

\section{Proofs for Large Deviation Inequalities}
\subsection{Proof of Claim~\ref{claim:phi-bounds}}
\label{sec:append-kernel-dev-1}
We show one inequality; the other one is bounded similarly.
$\forall k$, we compare the $k^{th}$ elements $\Phi_{2, k}, \Phi_{4, k}$ 
that appear in the sum for $\Phi_2$ and $\Phi_4$ respectively: 
\begin{eqnarray*}
\lefteqn{\frac{\Phi_{4, k}}{\Phi_{2, k}} =
\frac{(a_k^4 + b_k^4){4t^2}}{(a_k^2 + b_k^2){4t^4}}} \\
&= &
\left(\frac{2}{h} K\left(\frac{x_k - x_0}{h}\right) \sigma_{i}(x_k) \sigma_{j}(x_k)\right)^2 \cdot \\
&&
\frac{2\left((1+ \rho_{ij}(x_k) )^4 + (1- \rho_{ij}(x_k) )^4\right)}
{8(1+ \rho^2_{ij}(x_k) )} \\
& \leq & 
\max_{k}\left(\frac{2}{h} K\left(\frac{x_k- x_0}{h}\right) \sigma_{i}(x_k)  \sigma_{j}(x_k) \right)^2 \cdot \\
& &
\max_{0 \leq \rho \leq 1}
\frac{(1+ \rho)^4 + (1- \rho)^4}{4 (1+ \rho^2)} = 2 M^2. \; \; \; \square
\end{eqnarray*}
\subsection{Proof of Lemma~\ref{lemma:taylor-sums}}
\label{sec:append-taylor-sums}
We first use the Taylor expansions to obtain:
$$\ln \left(1 - a_k\right)
= - a_k - \frac{a_k^2}{2} - \frac{a_k^3}{3} 
- \frac{a_k^4}{4} -\sum^{\infty}_{l=5}\frac{(a_k)^{l}}{l},$$
where,
$$\sum^{\infty}_{l=5}\frac{(a_k)^{l}}{l} \leq
 \inv{5} \sum^{\infty}_{l=5}(a_k)^{5} = 
\frac{a_k^{5}}{5(1-a_k)} \leq \frac{2a_k^5}{5} \leq \frac{a_k^4}{5}$$
for $a_k < 1/2$; Similarly,
$$\ln \left(1 + b_k\right)
= \sum^{\infty}_{n=1}\frac{(-1)^{l-1}(b_k)^{l}}{l}, \; \mbox{ where}$$
$$\sum^{\infty}_{l=4}\frac{(-1)^{l}(b_k)^{l}}{l} > 0 \text{ and } 
\sum^{\infty}_{l=5}\frac{(-1)^{n}(b_k)^{l}}{l} < 0.$$
Hence for $b_k \leq a_k \leq \half, \forall k$,
\begin{eqnarray*}
\lefteqn{\half\sum_{k=1}^n \ln \inv{(1 - a_k)(1+b_k)}} \\
& \leq &
\sum_{k=1}^n \frac{a_k - b_k}{2} + \frac{a_k^2 + b_k^2}{4} + 
\frac{a_k^3 - b_k^3}{6} + \frac{9}{5} \frac{a_k^4 + b_k^4}{8} \\
& = & 
nt \Phi_1 + nt^2 \Phi_2 + nt^3 \Phi_3 + \frac{9}{5} nt^4 \Phi_4. \; \; \square
\end{eqnarray*}

\subsection{Proof of Claim~\ref{claim:Phi-2-bound}}
\label{sec:append-kernel-dev-2}
We replace the sum with the Riemann integral, and then use Taylor's Formula 
to replace $\sigma_{i}(x_k), \sigma_{j}(x_k)$, and $\sigma_{ij}(x_k)$,
\begin{eqnarray*}
\lefteqn{\Phi_2(i,j)  =  
\frac{1}{n}\sum_{k=1}^n
\frac{2}{h^2} K^2\left(\frac{x_k - x_0}{h}\right)
\left(\sigma_{i}^2(x_k) \sigma_{j}^2(x_k) + \sigma_{ij}^2(x_k) \right)}\\
& \approx & 
\int_{x_n}^{x_0} 
\frac{2}{h^2} K^2\left(\frac{u - x_0}{h}\right) 
\left(\sigma_{i}^2(u) \sigma_{j}^2(u) + \sigma_{ij}^2(u) \right) du \\
& = & 
\frac{2}{h}\int_{-\inv{h}}^{0}K^2(v) \left(\sigma_{i}^2(x_0 + hv) \sigma_{j}^2(x_0 + hv) 
+ \sigma_{ij}^2(x_0 + hv)\right) dv \\
& = & 
\nonumber
\frac{2}{h} \int_{-1}^{0} K^2(v)
\left(\sigma_{i}(x_0) + hv \sigma'_{i}(x_0) +
\frac{\sigma''_{i}(y_1)(hv)^2}{2}\right)^2
\\
&&
\left(\sigma_{j}(x_0) + hv \sigma'_{j}(x_0) + 
\frac{\sigma''_{j}(y_2)(hv)^2}{2}\right)^2 + \\ 
&& \left(\sigma_{ij}(x_0) + hv \sigma'_{ij}(x_0) +
\frac{\sigma''_{ij}(y_3)(hv)^2}{2}\right)^2 dv \\
& = &
\frac{2}{h}\int_{-1}^{0} K^2(v)
\left(
(1+ \rho^2_{ij}(x_0)) \sigma_{i}^2(x_0) \sigma_{k}^2(x_0)\right) dv + \\
&& C_2 \int_{-1}^{0} v K^2(v) dv +O(h)  \\
& = & 
\frac{C_1(1+ \rho^2_{ij}(x_0)) \sigma_{i}^2(x_0) \sigma_{j}^2(x_0)}{h}
\end{eqnarray*}
where $y_0, y_1, y_2 \leq hv + x_0$ and $C_1, C_2$ are some constants chosen
so that all equalities hold.
\; \; \; $\square$


\end{document}